\documentclass[lettersize,journal]{IEEEtran}
\usepackage{amsmath,amsfonts}
\usepackage{algorithmic}
\usepackage{algorithm}
\usepackage{array}
\usepackage[caption=false,font=normalsize,labelfont=sf,textfont=sf]{subfig}
\usepackage{textcomp}
\usepackage{stfloats}
\usepackage{url}
\usepackage{verbatim}
\usepackage{graphicx}
\usepackage{cite}
\usepackage[utf8]{inputenc} 
\usepackage[T1]{fontenc}    
\usepackage{hyperref}       
\usepackage{url}            
\usepackage{booktabs}       
\usepackage{amsfonts}       
\usepackage{nicefrac}       
\usepackage{microtype}      
\usepackage{lipsum}		
\usepackage{graphicx}
\usepackage{doi}

\usepackage{bm}
\usepackage{color}
\usepackage{amsmath}
\usepackage{amssymb}
\usepackage{diagbox}
\usepackage{lineno}
\usepackage{multirow}
\usepackage{mathtools}
\usepackage{arydshln}
\usepackage{enumerate}
\usepackage{enumitem}

\def\1{\mathbf{1}}
\def\0{\mathbf{0}}

\def\X{{\bf X}}

\def\x{{\bf x}}

\def\Y{{\bf Y}}
\def\y{{\bf y}}

\def\S{{\bf S}}
\def\s{{\bf s}}
\def\T{{\bf T}}

\def\I{{\bf I}}
\def\g{{\mathbf g}}
\def\q{{\mathbf q}}
\def\p{{\mathbf p}}

\def\Z{{\bf Z}}

\def\p{{\bf p}}

\def\H{{\bf H}}

\def\tr{{\mathrm{tr}}}
\def\nnz{{\mathbf{nnz}}}

\newtheorem{remark}{\bf Remark}
\newtheorem{definition}{\bf Definition}
\newtheorem{lemma}{\bf Lemma}
\newtheorem{theorem}{\bf Theorem}
\newtheorem{corollary}{\bf Corollary}
\newtheorem{proposition}{\bf Proposition}

\DeclarePairedDelimiter{\abs}{\lvert}{\rvert}
\DeclarePairedDelimiter{\norm}{\lVert}{\rVert}
\DeclarePairedDelimiter{\ceil}{\lceil}{\rceil}

\DeclarePairedDelimiter{\prn}{\lparen}{\rparen}
\DeclarePairedDelimiter{\brk}{\lbrack}{\rbrack}

\DeclarePairedDelimiter{\ang}{\langle}{\rangle}

\newenvironment{proof}{\hspace{0ex}\textbf{\textsc{Proof}}.\hspace{1ex}}{\hfill$\blacksquare$\newline}
\newcommand*\dif{\mathop{}\!\mathrm{d}}
\allowdisplaybreaks

\begin{document}

\title{Computationally Efficient Approximations for Matrix-based R\'enyi's Entropy}

\author{
    Tieliang~Gong$^{\dagger,*}$,
    Yuxin~Dong$^\dagger$,
    Shujian~Yu,
    Bo~Dong
\thanks{$^\dagger$These authors contribute equally.}
\thanks{$^*$Corresponding author: Tieliang Gong.}
\thanks{T. Gong (adidasgtl@gmail.com), Y. Dong (dongyuxin@stu.xjtu.edu.cn) and B. Dong (dong.bo@mail.xjtu.edu.cn) are with the School of Computer Science and Technology, Xi'an Jiaotong University and Shaanxi Provincial Key Laboratory of Big Data Knowledge Engineering, Xi'an 710049, China.}
\thanks{S. Yu (yusj9011@gmail.com) is with the Machine Learning Group, UiT - The Arctic University of Norway and Department of Computer Science, Vrije University Amsterdam, Amsterdam.}}

\markboth{IEEE Transactions on Signal Processing}%
{Gong \MakeLowercase{\textit{et al.}}: Computationally Efficient Approximations for Matrix-based R\'enyi's Entropy}


\maketitle

\begin{abstract}
The recently developed matrix-based R\'enyi’s $\alpha$-order entropy enables measurement of information in data simply using the eigenspectrum of symmetric positive semi-definite (PSD) matrices in reproducing kernel Hilbert space, without estimation of the underlying data distribution. This intriguing property makes this new information measurement widely adopted in multiple statistical inference and learning tasks. However, the computation of such quantity involves the trace operator on a PSD matrix $G$ to power $\alpha$ (i.e., $\tr(G^\alpha)$), with a normal complexity of nearly $\mathcal{O}(n^3)$, which severely hampers its practical usage when the number of samples (i.e., $n$) is large. In this work, we present computationally efficient approximations to this new entropy functional that can reduce its complexity to even significantly less than $\mathcal{O}(n^2)$. To this end, we leverage the recent progress on Randomized Numerical Linear Algebra, developing Taylor, Chebyshev and Lanczos approximations to $\tr(G^\alpha)$ for arbitrary values of $\alpha$ by converting it into a matrix-vector multiplication problem. We also establish the connection between the matrix-based R\'enyi’s entropy and PSD matrix approximation, which enables exploiting both clustering and block low-rank structure of $G$ to further reduce the computational cost. We theoretically provide approximation accuracy guarantees and illustrate the properties for different approximations. Large-scale experimental evaluations on both synthetic and real-world data corroborate our theoretical findings, showing promising speedup with negligible loss in accuracy.
\end{abstract}


\section{Introduction}
\IEEEPARstart{I}{nformation-theoretic} measures (such as entropy, mutual information, synergy) and principles (such as information bottleneck~\cite{tishby1999information} and maximum entropy~\cite{jaynes1957information,jaynes1957information2}) have a long track record of usefulness in machine learning and neuroscience~\cite{principe2010information,victor2006approaches,timme2018tutorial}. Notable examples include feature selection by maximizing the mutual information between selected feature subset and class labels~\cite{battiti1994using} and the utilization of partial information decomposition (PID) framework~\cite{williams2010nonnegative} to quantify the neural goal functions~\cite{wibral2017partial}. In the deep learning area, information-theoretic methods have become the workhorse of several impressive deep learning breakthroughs in recent years, ranging from the representation learning by variational information bottleneck~\cite{alemi2016deep} or information maximization (InfoMax)~\cite{RDevonHjelm2018LearningDR} to theoretical investigations on the generalization bound~\cite{xu2017information}.

The combination of information theory with machine learning and neuroscience usually involves an exact access to different information-theoretic measures defined over the probability space. However, the accurate probability density function (PDF) estimation is computationally infeasible in high-dimensional space~\cite{fan2006statistical}, which impedes more widespread adoption of information-theoretic methods in data-driven science. 

In~\cite{giraldo2014measures,giraldo2014information} and later~\cite{yu2019multivariate}, a novel information measure named the matrix-based R\'enyi's $\alpha$-order entropy was proposed, enabling us to quantify the information of a single variable or complex interactions across multiple variables directly from given samples. Distinct from the classical Shannon entropy~\cite{shannon1948mathematical} and R\'enyi's entropy~\cite{renyi1961measures}, this new family of information measures is defined on the eigenspectrum of a Gram matrix constructed by projecting data in reproducing kernel Hilbert space (RKHS), thus avoiding the expensive evaluation of underlying distributions. This elegant property makes the matrix-based entropy functional being a reliable choice in lots of machine learning and neuroscience applications, which include but are not limited to similarity measurement~\cite{brockmeier2017quantifying}, feature reduction~\cite{alvarez2017kernel}, compressing deep neural networks by neuron pruning~\cite{sarvani2022hrel}, identifying the most informative regions in the brain (i.e., subgraph) to psychiatric disorders (such as depression and autism)~\cite{zheng2022brainib}, inferring the effective connectivity across different brain areas~\cite{de2019data}.

Despite of its practical utility, exactly computing matrix-based R\'enyi's entropy involves calculating eigenvalues of a Gram matrix $G$ of size $n \times n$ ($n$ is the sample size, please refer to Definition \ref{Renyi}), which requires $\mathcal{O}(n^3)$ computational complexity with traditional techniques such as Singular Value Decomposition (SVD) \cite{elad2010sparse, li2014large}, CUR decomposition \cite{mahoney2009cur, mahoney2011randomized} or QR Factorization \cite{watkins2008qr, watkins2004fundamentals}. This drawback poses great challenges for both storage and computing in practice. We therefore seek for computationally efficient methods with statistical guarantees to approximate the matrix-based R\'enyi's entropy. To this end, we borrow the idea of stochastic trace estimation~\cite{avron2011randomized} for integer $\alpha$-order entropy approximation, where the trace estimation problem is transferred into matrix-vector multiplications. We then employ polynomial approximation techniques to estimate the power of $G$ for fractional-order $\alpha$. Furthermore, we establish the connection between matrix-based R\'enyi’s entropy estimation and Gram matrix approximation, where the $k$-means algorithm is employed to discover the block structure of $G$ and a low-rank decomposition is conducted to approximate the off-diagonal blocks. In summary, our contributions are four-fold: 
\begin{itemize}
    \item We develop Taylor, Chebyshev and Lanczos approximations for arbitrary $\alpha$-order (integer \& fractional) matrix-based R\'enyi's entropy that transforms the original trace estimation problem into matrix-vector multiplications, reducing the overall complexity to roughly $\mathcal{O}(n^2s)$, where $s \ll n$ is the number of queried random vector.
    \item We additionally reveal the intrinsic connection between eigenspectrum estimation and Gram matrix approximation. To this end, we investigate the spectral structure of the Gram matrix and design a novel block low-rank approximation method, which enables us to further reduce the computational burden to $\mathcal{O}(n^2/c + nck)$, where $c \ll n$ is the number of clusters and $k \ll n$ is the order of rank.
    \item Theoretically, we conduct statistical analysis for each of the estimators above, establishing accuracy guarantees and revealing their intrinsic properties and connections.
    \item We evaluate the effectiveness of the proposed approximations on large-scale simulation and real-world datasets, showing remarkable acceleration on various information-related tasks with no significant loss in accuracy.
\end{itemize}


\section{Preliminaries} \label{sec:pre}
Entropy measures the uncertainty of random variables using a single scalar quantity \cite{klir2013uncertainty}. For a continuous random variable $X$ with PDF $p(\x)$ defined on the finite set $\mathcal{X}$, the $\alpha$-order R\'enyi's entropy ($\alpha > 0$ and $\alpha \ne 1$) $\H_\alpha(X)$ is defined as
\begin{equation} \label{eq_renyi}
    \H_{\alpha}(X) = \frac{1}{1 - \alpha} \log \int_{\mathcal{X}} p^\alpha(\x) \dif \x.
\end{equation}
The limit case of $\alpha \rightarrow 1$ yields the well known Shannon's entropy, i.e. $\H(X) = -\int_{\mathcal{X}} p(\x) \log p(\x) \dif \x$. It is easy to see that R\'enyi's entropy is a natural extension of Shannon entropy by introducing a hyper-parameter $\alpha$. In real-world applications, the choice of $\alpha$ is task-specific: on the one hand, $\alpha$ should be less than $2$ or even $1$ when the learning task requires estimating tails of the distribution or multiple modalities; on the other hand, $\alpha$ is suggested to be greater than $2$ to emphasize mode behavior when we aim to characterize the mean behavior \cite{principe2010information}.

It is worth noting that calculating R\'enyi $\alpha$-order entropy requires the prior about the PDF of data, which prevents its more widespread adoption in data-driven science. To alleviate this issue, an alternative measure, namely the matrix-based R\'enyi's $\alpha$-order entropy was recently proposed, which resembles quantum R\'enyi’s entropy in terms of the eigenspectrum of a normalized Hermitian matrix constructed by projecting data in RKHS. It is defined as:
\begin{definition}~\cite{giraldo2014measures} \label{Renyi}
    Let $\varphi: \mathcal{X} \times \mathcal{X} \mapsto \mathbb{R}$ be a real valued positive kernel that is also infinitely divisible \cite{bhatia2006infinitely}. Given $ \{\x_i \}_{i=1}^n \subset \mathcal{X}$, each $\x_i$ being a real-valued scalar or vector, and the Gram matrix $K$ obtained from $K_{ij} = \varphi(\x_i, \x_j)$, a matrix-based analogue to R\'enyi's $\alpha$-entropy can be defined as:
    \begin{equation} \label{eq:renyi_entropy}
    	\S_\alpha(G) = \frac{1}{1-\alpha}\log_2(\tr(G^\alpha)) = \frac{1}{1-\alpha}\log_2\brk*{\sum_{i=1}^n \lambda_i^\alpha(G)},
    \end{equation}
    where $G_{ij} = \frac{1}{n}\frac{K_{ij}}{\sqrt{K_{ii}K_{jj}}}$ is a normalized kernel matrix and $\lambda_i(G)$ denotes the $i$-th eigenvalue of $G$.
\end{definition}

Let $\Delta_n^+$ be the set of positive semi-definite matrices of size $n \times n$ whose elements take real values. From \cite{giraldo2014measures} we know that $\S_\alpha(PGP^*) = \S_\alpha(G)$ for any orthonormal matrix $P$, and that $\S_\alpha(\frac{1}{n} I) = \log_2(n)$ for identity matrix $I \in \Delta_n^+$ takes the maximum entropy value among all $n \times n$ kernel matrices. We denote the maximum and minimum eigenvalue of $G$ as $\lambda_{max}$ and $\lambda_{min}$ respectively, the condition number is then $\kappa = \lambda_{max}/\lambda_{min}$.

\begin{definition}~\cite{yu2019multivariate}
    Given a collection of $n$ samples $\{\s_i=(\x_i^1,\x_i^2,\cdots, \x_i^L)\}_{i=1}^n$, each sample contains $L$ ($L\geq2$) measurements $\x^1\in \mathcal{X}^1$, $\x^2\in \mathcal{X}^2$, $\cdots$, $\x^L\in \mathcal{X}^L$ obtained from the same realization, and the positive definite kernels $\varphi_1:\mathcal{X}^1\times \mathcal{X}^1\mapsto\mathbb{R}$, $\varphi_2:\mathcal{X}^2\times \mathcal{X}^2\mapsto\mathbb{R}$, $\cdots$, $\varphi_L:\mathcal{X}^L\times \mathcal{X}^L\mapsto\mathbb{R}$, a matrix-based analogue to R{\'e}nyi's $\alpha$-order joint-entropy among $L$ variables can be defined as:
    \begin{equation}\label{eq:renyi_joint}
    	\mathbf{S}_\alpha(G^1,G^2,\cdots, G^L)=\mathbf{S}_\alpha\left(\frac{G^1\circ G^2\circ\cdots\circ G^L}{\mathrm{tr}(G^1\circ G^2\circ\cdots\circ G^L)}\right),
    \end{equation}
    where $(G^1)_{ij}=\varphi_1(\x_i^1,\x_j^1)$, $(G^2)_{ij}=\varphi_2(\x_i^2,\x_j^2)$, $\cdots$, $(G^L)_{ij}=\varphi_L(\x_i^L,\x_j^L)$, and $\circ$ denotes the Hadamard product.
\end{definition}
Given Eq.~(\ref{eq:renyi_entropy}) and (\ref{eq:renyi_joint}), the matrix-based R\'enyi's $\alpha$-order mutual information $I_\alpha(\cdot; \cdot)$ and total correlation $T_\alpha(\cdot)$ amongst $\x^1, \x^2, \cdots, \x^L$ and $\y$ in analogy of Shannon's definitions can be derived accordingly:
\begin{align}
    \I_\alpha(G^1, \cdots, G^L; G) =&\, \S_\alpha(G^1, \cdots, G^L) + \S_\alpha(G) \nonumber\\
    &- \S_{\alpha}(G^1, \cdots, G^L, G). \label{eq:Renyi_MI} \\
    \T_\alpha(G^1, \cdots, G^L) =&{} \brk*{\sum_{i=1}^{L} {\S_\alpha(G^i)}} - \S_{\alpha}(G^1, \cdots, G^L). \label{eq:Renyi_TC}
\end{align}
It is simple to verify that both of them are permutation-invariant to the ordering of kernel matrices $G^1, \cdots, G^L$. Furthermore, the matrix-based entropy functionals are independent of the specific dimensions of variables $\mathbf{x}^1, \cdots, \mathbf{x}^L$, thus avoiding estimation of the underlying data distributions and making them suitable to be applied on high-dimensional data with either discrete or continuous distributions.

\section{Approximating Matrix-based R\'enyi's Entropy: Randomized Linear Algebra Approaches}
In this section, we present computationally efficient approaches with statistical guarantees for matrix-based R\'enyi's entropy approximation from the perspective of randomized numerical linear algebra, where the general idea is to approximate the trace operator by random matrix-vector multiplications to avoid the expensive eigenvalue decomposition.

\subsection{Randomized Approximation}
Inspired by the work on randomized trace estimation \cite{avron2011randomized}, we adopt random algorithms for calculating the matrix-based R\'enyi's entropy. The following lemma \cite{roosta2015improved} characterizes an algorithm that approximates the trace of any symmetric positive semi-definite (PSD) matrix by computing inner products of the matrix with Gaussian random vectors:
\begin{lemma} \cite{avron2011randomized}
    \label{TraceEstim}
    Let $G \in \mathbb{R}^{n \times n}$ be a symmetric positive semi-definite matrix. If $\g_1$, $\g_2$, $\cdots$, $\g_s \in \mathbb{R}^n$ are i.i.d random standard Gaussian vectors, then for $s = \ceil{8\ln(2/\delta)/\epsilon^2}$, with probability at least $1-\delta$,
    \begin{equation*}
    	\abs*{\tr(G)-\frac{1}{s}\sum_{i=1}^s \g_i^\top G \g_i} \le \epsilon \cdot \tr(G).
    \end{equation*}
\end{lemma}
Lemma \ref{TraceEstim} immediately implies Algorithm \ref{alg:IntAlgo} for integer order matrix-based R\'enyi's entropy estimation. With conventional eigenvalue-based approaches, exactly calculating the matrix-based R\'enyi's entropy generally requires $\mathcal{O}(n^3)$ time complexity. Algorithm \ref{alg:IntAlgo} successfully transforms the eigenvalue problem into matrix-vector multiplications whose complexity is $\mathcal{O}(\alpha s \cdot \mathbf{nnz}(G))$, where $s \ll n$ is the number of random vector queries and $\mathbf{nnz}(G)$ denotes the number of non-zero elements in $G$ ($\mathbf{nnz}(G) \approx n^2$ for dense matrix $G$). Theorem \ref{IntBound} provides the quality-of-approximation result for Algorithm \ref{alg:IntAlgo}. 

\begin{algorithm}[t]
    \caption{Integer order matrix-based R\'enyi's entropy estimation}
    \label{alg:IntAlgo}
    \begin{algorithmic}[1]
    	\STATE \textbf{Input:} Kernel matrix $G \in \mathbb{R}^{n \times n}$, number of random vectors $s$, integer order $\alpha \ge 2$.
    	\STATE \textbf{Output:} Approximation to $S_\alpha(G)$.		
    	\STATE Generate $s$ independent random standard Gaussian vectors $\g_i, i = 1, \cdots, s$.
    	\STATE \textbf{Return:} $\tilde{\S}_\alpha(G) = \frac{1}{1-\alpha} \log_2(  \frac{1}{s} \sum_{i=1}^s \g_i^\top G^\alpha \g_i)$.
    \end{algorithmic}
\end{algorithm}
\begin{theorem}
    \label{IntBound}
    Let $G \in \mathbb{R}^{n \times n}$ be a normalized PSD kernel matrix and $\tilde{\S}_\alpha(G)$ be the output of Algorithm \ref{alg:IntAlgo} with $s = \ceil{8\ln(2/\delta)/\epsilon^2}$. Then with confidence at least $1-\delta$,
    \begin{equation*}
    	\abs*{\S_\alpha(G) - \tilde{\S}_\alpha(G)} \le  \abs*{\frac{1}{1-\alpha} \log_2(1 - \epsilon)}.
    \end{equation*}
\end{theorem}
\begin{proof}
    Let $\tilde{\tr}(G^\alpha) = \frac{1}{s}\sum_{i=1}^s \g_i^\top G^\alpha \g_i$, then
    \begin{align*}
        \abs*{\S_\alpha(G) - \tilde{\S}_\alpha(G)} &= \abs*{\frac{1}{1-\alpha} \log_2\prn*{\frac{\tilde{\tr}(G^\alpha)}{\tr(G^\alpha)}}} \\
        &= \abs*{\frac{1}{1-\alpha} \log_2\prn*{1 - \frac{\tr(G^\alpha) - \tilde{\tr}(G^\alpha)}{\tr(G^\alpha)}}} \\
        &\le \abs*{\frac{1}{1-\alpha} \log_2\prn*{1 - \abs*{\frac{\tr(G^\alpha) - \tilde{\tr}(G^\alpha)}{\tr(G^\alpha)}}}} \\
        &\le \abs*{\frac{1}{1-\alpha} \log_2\prn*{1 - \epsilon}}.
    \end{align*}
    Where the first inequality follows by the fact that $\abs{\log_2(1-x)} \le -\log_2(1-\abs{x})$ for all $x \in (-1, 1)$, and the second inequality follows by Lemma \ref{TraceEstim}.
\end{proof}
\begin{remark}
    Theorem \ref{IntBound} theoretically proves the feasibility of approximating matrix-based R\'enyi's entropy through matrix-vector multiplication operations, which leads to substantially lower computational cost than eigenvalue-based methods. The approximation quality depends on the number of queried random vectors $s$ and hyper-parameter $\alpha$. In particular, with the increase of $\alpha$, the approximation error decreases for integer $\alpha \ge 2$.
\end{remark}

\subsection{Taylor Series Approximation}
The above approach could only handle integer $\alpha$-orders. For fractional $\alpha$ cases, the result of matrix-vector multiplications $G^\alpha \cdot \g_i$ cannot be directly acquired in the same manner. To address this issue, we apply a Taylor series expansion on the matrix $\alpha$-power functional, where the implementing details are presented in Algorithm \ref{alg:Taylor}. It requires extra estimation of $\lambda_{max}$, the dominant eigenvalue of $G$. Here, we adopt power iteration \cite{boutsidis2017randomized}, an estimator for extreme matrix eigenvalues via random vector queries. As a result, matrix-vector multiplication is still the only operator that directly accesses the kernel matrix $G$ in our algorithms, enabling further optimizations with matrix approximation techniques. The total computational complexity of Algorithm \ref{alg:Taylor} is $\mathcal{O}(ms \cdot \mathbf{nnz}(G))$, where $m$ is the degree of the Taylor series.

\begin{algorithm}[t]
    \caption{Fractional order matrix-based R\'enyi's entropy estimation via Taylor series}
    \label{alg:Taylor}
    \begin{algorithmic}[1]
    	\STATE \textbf{Input:} Kernel matrix $G \in \mathbb{R}^{n \times n}$, number of random vectors $s$, fractional order $\alpha > 0$, polynomial degree $m > \alpha$.
    	\STATE \textbf{Output:}  Approximation to $S_\alpha(G)$.
    	\STATE Calculate $\lambda_{max}$ by power iteration.
    	\STATE Generate $s$ independent random standard Gaussian vectors $\g_i, i = 1, \cdots, s$.
    	\STATE \textbf{Return:} $\tilde{\S}_\alpha(G) = \frac{1}{1-\alpha} \cdot$\\
    	$\qquad\log_2 \prn*{ \frac{\lambda_{max}^\alpha}{s}\sum_{i=1}^s \sum_{n=0}^m \binom{\alpha}{n} \g_i^\top\prn*{\frac{G}{\lambda_{max}} - I_n}^n \g_i}$.
    \end{algorithmic}
\end{algorithm}
\begin{theorem}
    \label{th:TaylorBound}
    Let $G \in \mathbb{R}^{n \times n}$ be a normalized PSD kernel matrix and $\tilde{\S}_\alpha(G)$ be the output of Algorithm \ref{alg:Taylor} with $s = \ceil{8\ln(2/\delta)/\epsilon^2}$ and
    \begin{align}\label{eq:taylor_order}
    	m = \ceil*{\alpha + 
    	\begin{cases}
    		\left.W_0\prn*{\kappa\beta \sqrt[\alpha+1]{\frac{\Gamma(\alpha+1)}{\epsilon\pi}}}\middle/\beta\right. & \textrm{if} ~~\lambda_{min} > 0\\
    		\sqrt[\alpha]{\frac{n\Gamma(\alpha+1)}{\epsilon\pi}} &\textrm{if} ~~ \lambda_{min} = 0
    	\end{cases}}.
    \end{align}
    Then with confidence at least $1-\delta$,
    \begin{equation*}
    	\abs*{\S_\alpha(G) - \tilde{\S}_\alpha(G)} \le \abs*{\frac{1}{1-\alpha} \log_2(1- 3\epsilon)},
    \end{equation*}
    where $\beta = -\log(1-1/\kappa)/(\alpha+1)$ and $W_0$ is the principal branch of Lambert $W$ function.
\end{theorem}
\begin{proof}
    In algorithm \ref{alg:Taylor}, we use binomial series, the Taylor expansion for function $f(x) = (x+1)^\alpha$:
    \begin{equation}
    	(1+x)^\alpha = \sum_{i=0}^\infty \binom{\alpha}{i}x^i.
    	\label{BinSeries}
    \end{equation}
    Binomial series converge absolutely for symmetric matrices $G$ with eigenvalues in $[-1,1)$ and $\alpha>0$. Taking $\lambda_{max}$, the maximum eigenvalue of $G$, the eigenvalues of $\frac{G}{\lambda_{max}}-I$ are within $[-1,0]$, which means the series above converge absolutely. Therefore, for each eigenvalue $\lambda$ of $G$:
    \begin{equation}
    	\lambda^\alpha = \lambda_{max}^\alpha \cdot \prn*{\frac{\lambda}{\lambda_{max}}}^\alpha = \lambda_{max}^\alpha \sum_{i=0}^\infty \binom{\alpha}{i}\prn*{\frac{\lambda}{\lambda_{max}}-1}^i.
    	\label{TaylorSeries}
    \end{equation}
    Let $f_m(\lambda)=\lambda_{max}^\alpha\sum_{i=0}^m \binom{\alpha}{i}(\lambda/\lambda_{max}-1)^i$ denote the Taylor series above with degree $m$. When $\lambda_{min} > 0$, we have
    \begin{align}
        \abs{\lambda^\alpha - f_m(\lambda)} &= \lambda_{max}^\alpha\abs*{\sum_{i=m+1}^\infty \binom{\alpha}{i} \prn*{\frac{\lambda}{\lambda_{max}}-1}^i} \nonumber\\
        &\le \lambda_{max}^\alpha\abs*{\binom{\alpha}{m+1} \sum_{i=m+1}^\infty \prn*{\frac{\lambda}{\lambda_{max}}-1}^i} \label{PosDef}\\
        &= \lambda_{max}^\alpha\abs*{\frac{\lambda_{max}}{\lambda}\binom{\alpha}{m+1} \prn*{\frac{\lambda}{\lambda_{max}}-1}^{m+1}} \nonumber\\
        &\le \lambda_{max}^\alpha\kappa\abs*{\binom{\alpha}{m+1}\prn*{\frac{1}{\kappa}-1}^{m+1}}. \nonumber
    \end{align}
    Recall that the binomial term satisfies $\binom{\alpha}{i} = \prod_{j=1}^i \frac{\alpha-j+1}{j}$. When $i > \alpha$, $\binom{\alpha}{i+1} = \binom{\alpha}{i}\frac{\alpha-i}{i+1}$, which means $\binom{\alpha}{i+1} \cdot \binom{\alpha}{i} \le 0$. Combining with the fact that $\lambda/\lambda_{max}-1 \le 0$, we know that the coefficients $\binom{\alpha}{i}\prn{\frac{\lambda}{\lambda_{max}}-1}^i$ above share the same sign for $i > \alpha$. By assuming $m > \alpha$, (\ref{PosDef}) follows by noticing that $\abs{\binom{\alpha}{i}}=\abs{\binom{\alpha}{m+1}\prod_{j=m+2}^i\frac{\alpha-j+1}{j}}\le\abs{\binom{\alpha}{m+1}}$, $\forall i \ge m+1$, since $\abs{\frac{\alpha-j+1}{j}} \le 1$ when $j \ge \alpha+1$.
    
    To upper bound the binomial term $\binom{\alpha}{m+1}$, we introduce the following lemma:
    \begin{lemma}
    	\label{GammaIneq}
    	\cite{das2019} Let $\Gamma(x)$ be the gamma function and let $R(x,y) = \Gamma(x+y)/\Gamma(x)$, then
    	\begin{align*}
    		R(x, y) &\ge x^y &for\enspace 1\le y \le 2,\\
    		R(x, y) &\ge x(x+1)^{y-1} &for\enspace y \ge 2.
    	\end{align*}
    \end{lemma}
    Then we have the following upper bound:
    \begin{align}
    	\abs*{\binom{\alpha}{m+1}} &= \abs*{\frac{\Gamma(\alpha+1)}{\Gamma(m+2)\Gamma(\alpha-m)}} \nonumber\\
    	&\le \abs*{\frac{\Gamma(\alpha+1)}{\Gamma(m-\alpha+1)\Gamma(\alpha-m)}} \nonumber\\
    	&\qquad \cdot 
    	\begin{cases}
    		\frac{1}{(m-\alpha+1)^{\alpha+1}} & 0 < \alpha < 1\\
    		\frac{1}{(m-\alpha+1)(m-\alpha+2)^\alpha} & \alpha > 1
    	\end{cases}
    	\label{GammaIneq1}\\
    	&\le \frac{\Gamma(\alpha+1)}{\pi(m-\alpha+1)} \frac{1}{(m-\alpha+1)^\alpha} \label{GammaProp1} \\
    	&\le \frac{\Gamma(\alpha+1)}{\pi(m-\alpha+1)^{\alpha+1}}. \nonumber
    \end{align}
    (\ref{GammaIneq1}) follows by applying Lemma \ref{GammaIneq} on $R(m-\alpha+1, \alpha+1)$. (\ref{GammaProp1}) follows by Euler's reflection formula that for any non-integer number $z$, $\Gamma(z)\Gamma(1-z)=\pi/\sin(\pi z)$. By limiting the choice of $m$, we would like to ensure that
    \begin{align*}
        \abs*{\lambda^\alpha - f_m(\lambda)} &\le \lambda_{max}^\alpha \kappa \abs*{\frac{\Gamma(\alpha+1)}{\pi(m-\alpha+1)^{\alpha+1}} \prn*{\frac{1}{\kappa}-1}^{m+1}} \\
        &\le \epsilon\lambda_{min}^\alpha \le \epsilon\lambda^\alpha.
    \end{align*}
    Solving the inequation above yields:
    \begin{equation*}
    	m \ge \alpha + \left.W_0\prn*{\kappa\beta \sqrt[\alpha+1]{\frac{\Gamma(\alpha+1)}{\epsilon\pi}}}\middle/\beta\right.,
    \end{equation*}
    where $\beta = -\log(1-1/\kappa)/(\alpha+1)$ and $W_0$ is the principal branch of Lambert $W$ function. Combining the results above, we have that with confidence at least $1 - \delta$:
    \begin{align}
    	&\quad\abs*{\sum_{i=1}^s \g_i^\top f_m(G) \g_i - \mathbf{tr}(G^\alpha)} \nonumber\\
    	&\le \abs*{\sum_{i=1}^s \g_i^\top f_m(G) \g_i - \mathbf{tr}(f_m(G))} + \abs*{\mathbf{tr}(f_m(G)) - \mathbf{tr}(G^\alpha)} \nonumber\\
    	&\le \epsilon \cdot \mathbf{tr}(f_m(G)) + \abs*{\mathbf{tr}(f_m(G)) - \mathbf{tr}(G^\alpha)} \label{TaylorTrace}\\
    	&\le \epsilon \cdot \mathbf{tr}(G^\alpha) + (1+\epsilon) \cdot \abs*{\mathbf{tr}(f_m(G)) - \mathbf{tr}(G^\alpha)} \nonumber\\
    	&= \epsilon \cdot \mathbf{tr}(G^\alpha) + (1+\epsilon) \cdot \sum_{i=1}^n\abs*{\lambda_i^\alpha - f_m(\lambda_i)} \nonumber\\
    	&\le \epsilon \cdot \mathbf{tr}(G^\alpha) + \epsilon (1+\epsilon) \cdot \sum_{i=1}^n \lambda_i^\alpha \nonumber\\
    	&\le 3\epsilon \cdot \mathbf{tr}(G^\alpha), \label{PolyBound}
    \end{align}
    where (\ref{TaylorTrace}) follows by applying Lemma \ref{TraceEstim} on matrix $f_m(G)$. This finishes the proof for the case $\lambda_{min} > 0$ by adopting similar steps in the proof of Theorem \ref{IntBound}.
    
    Elsewise when $\lambda_{min} = 0$, notice that
    \begin{align}
        &\quad \prn*{1+\frac{\lambda}{\lambda_{max}}-1}\sum_{i=m}^\infty \binom{\alpha}{i} \prn*{\frac{\lambda}{\lambda_{max}}-1}^i \nonumber\\
        &= \sum_{i=m}^\infty \binom{\alpha}{i} \prn*{\frac{\lambda}{\lambda_{max}}-1}^i + \sum_{i=m+1}^\infty \binom{\alpha}{i-1} \prn*{\frac{\lambda}{\lambda_{max}}-1}^i \nonumber\\
        &= \binom{\alpha}{m} \prn*{\frac{\lambda}{\lambda_{max}}-1}^m + \sum_{i=m+1}^\infty \binom{\alpha+1}{i} \prn*{\frac{\lambda}{\lambda_{max}}-1}^i,
    \end{align}
    where the last step follows by the property of binomial terms that for any $\alpha > 0$ and integer $i > 1$, $\binom{\alpha}{i-1} + \binom{\alpha}{i} = \binom{\alpha+1}{i}$. Then by setting $\lambda = 0$ in the equation above, we have
    \begin{equation*}
        \binom{\alpha}{m} (-1)^m + \sum_{i=m+1}^\infty \binom{\alpha+1}{i} (-1)^i = 0.
    \end{equation*}
    Through similar steps as the proof above, we have
    \begin{align*}
        \abs*{\lambda^\alpha - f_m(\lambda)} &= \lambda_{max}^\alpha \abs*{\sum_{i=m+1}^\infty \binom{\alpha}{i} \prn*{\frac{\lambda}{\lambda_{max}}-1}^i} \\
        &\le \lambda_{max}^\alpha \abs*{\sum_{i=m+1}^\infty \binom{\alpha}{i} (-1)^i} \\
        &= \lambda_{max}^\alpha \abs*{-\binom{\alpha-1}{m} (-1)^m}  \nonumber\\
    	&\le \frac{\lambda_{max}^\alpha \Gamma(\alpha)}{\pi} \cdot \frac{1}{(m-\alpha+1)^\alpha}.
    \end{align*}
    Then by choosing $m$ as:
    \begin{equation*}
    	m \ge \alpha + \sqrt[\alpha]{\frac{n\Gamma(\alpha+1)}{\epsilon\pi}},
    \end{equation*}
    we have $\abs{\lambda^\alpha-f_m(\lambda)} \le \frac{\epsilon}{n}\lambda_{max}^\alpha$ and
    \begin{equation*}
        \sum_{i=1}^n\abs*{\lambda_i^\alpha - f_m(\lambda_i)} \le \epsilon\lambda_{max}^\alpha \le \epsilon \cdot \tr(G^\alpha),
    \end{equation*}
    which finishes the proof for the case $\lambda_{min} = 0$.
\end{proof}
\begin{remark}
    Theorem \ref{th:TaylorBound} establishes the statistical guarantee for Algorithm \ref{alg:Taylor}, whose approximation error relies on the condition number $\kappa$ of $G$, the number of random vectors $s$ and the degree of Taylor series $m$. Note that $s = \mathcal{O}(1/\epsilon^2)$ is required for both full-rank and rank-deficient cases, but the requirement for $m$ is slightly different. As shown in (\ref{eq:taylor_order}), $m = \mathcal{O}\prn*{\alpha + \kappa\log\prn*{\frac{\Gamma(\alpha+1)}{\epsilon(\alpha+1)^{\alpha+1}}+1}}$ (by the fact that $W_0(z) \leq \log(1+z), z\geq 0$) is sufficient to yield a meaningful approximation for full rank cases, less than the requirement of rank-deficient cases where $m = \mathcal{O} \prn*{\alpha + \prn*{\frac{n\Gamma(\alpha+1)}{\epsilon}}^{\frac{1}{\alpha}}}$.
\end{remark}

\subsection{Chebyshev Series Approximation}
Chebyshev series leverage orthogonal polynomials to approximate arbitrary analytic functions. It serves as an alternative approach for polynomial approximation, which usually enjoys faster convergence rates compared to Taylor series. Let $T_k (x) = \cos(k \arccos (2x/\lambda_{max} - 1))$ with $x \in [0, \lambda_{max}]$ be the Chebyshev polynomials of the first kind for any integer $k \ge 0$, and let $f_m(x) = c_0 T_0(x)/2 + \sum_{k=1}^{m} c_k T_k(x)$ with
\begin{equation} \label{coeff}
    c_k = \frac{2\lambda_{max}^\alpha \Gamma(\alpha + 1/2) (\alpha)_k}{\sqrt{\pi} \Gamma(\alpha + 1) (\alpha+k)_k},
\end{equation}
where $(\alpha)_k = \alpha \cdot (\alpha - 1) \cdots (\alpha - k + 1)$ is the falling factorial. Algorithm \ref{alg:ChebyAlgo} summarizes the procedure of approximating  fractional order matrix-based R\'enyi's entropy through Chebyshev series, in which Clenshaw's algorithm \cite{clenshaw1955note} is employed to calculate $F_i(G) = \g_i^\top f_m(G) \g_i$.

\begin{algorithm}
    \caption{Fractional order matrix-based R\'enyi's entropy estimation via Chebyshev series}
    \label{alg:ChebyAlgo}
    \begin{algorithmic}[1]
    	\STATE \textbf{Input:} Kernel matrix $G \in \mathbb{R}^{n \times n}$, number of random vectors $s$, fractional order $\alpha > 0$, polynomial degree $m > \alpha$.
    	\STATE \textbf{Output:}  Approximation to $S_\alpha(G)$.
    	\STATE Calculate $\lambda_{max}$ by power iteration.
    	\STATE Calculate Chebyshev coefficients $c_k$, $k \in [0, m]$ by Eq.~(\ref{coeff}).
    	\STATE Generate $s$ independent random standard Gaussian vectors $\g_i, i = 1, \cdots, s$.
    	
    	\FOR{$i=1,2, \cdots, s$}
    	\STATE 	Set $\y_{m+2} = \y_{m+1} = \0$.
    	\FOR{$k = m, m-1, \cdots, 0$}
    	\STATE 	$\y_k = c_k \g_i + 4G \y_{k+1} /\lambda_{max} - 2\y_{k+1} - \y_{k+2}$.
    	\ENDFOR
    	\STATE 	Calculate $F_i(G) = c_0 \g_i^\top \g_i /4 + \g_i^\top (\y_0 - \y_2)/2$.
    	\ENDFOR
    	
    	\STATE \textbf{Return:} $\tilde{\S}_\alpha(G) = \frac{1}{1-\alpha} \log_2 \Big(\frac{1}{s} \sum_{i=1}^s F_i(G)\Big)$.
    \end{algorithmic}
\end{algorithm}
\begin{theorem}
    \label{th:ChebyBound}
    Let $G \in \mathbb{R}^{n \times n}$ be a normalized PSD kernel matrix and $\tilde{\S}_\alpha(G)$ be the output of Algorithm \ref{alg:ChebyAlgo} with $s = \ceil{8\ln(2/\delta)/\epsilon^2}$ and
    \begin{equation*}
    	m = \ceil*{\alpha + 
    	\begin{cases}
    		\sqrt{\kappa} \sqrt[2\alpha]{\frac{\Gamma(\alpha + \frac{1}{2}) \Gamma(\alpha)}{\epsilon\pi^{3/2}}} & \textrm{if}~~ \lambda_{min} > 0\\
    		\sqrt[2\alpha]{\frac{n\Gamma(\alpha + \frac{1}{2}) \Gamma(\alpha)}{\epsilon\pi^{3/2}}} & \textrm{if}~~ \lambda_{min} = 0
    	\end{cases}}.
    \end{equation*}
    Then with confidence at least $1-\delta$, 
    \begin{equation*}
    	\abs*{\S_\alpha(G) - \tilde{\S}_\alpha(G)} \le \abs*{\frac{1}{1-\alpha} \log_2(1 - 3\epsilon)}.
    \end{equation*}
    Please refer to the appendix for the proof.
\end{theorem}
\begin{remark}
    Theorem \ref{th:ChebyBound} establishes the statistical guarantees of Algorithm \ref{alg:ChebyAlgo} in terms of $s$, $\kappa$ and $m$. Similar with Taylor approximation, $s = \mathcal{O}(1/\epsilon^2)$ is required for both full rank and rank-deficient cases, where the corresponding $m$ is $\mathcal{O}\prn*{\alpha+\sqrt{\kappa}\prn*{\frac{\Gamma(2\alpha)}{\epsilon} }^\frac{1}{2\alpha}}$ and $\mathcal{O}\prn*{\alpha+\prn*{\frac{n\Gamma(2\alpha)}{\epsilon}}^\frac{1}{2\alpha}}$ respectively. Compared with Theorem \ref{th:TaylorBound}, Chebyshev approximation requires substantially lower degree than Taylor approximation especially when $\kappa$ is large.
\end{remark}

\subsection{Lanczos Quadrature Approximation}
Besides explicit polynomial expansions, an alternative approach for estimating implicit matrix-vector multiplications is the Lanczos method \cite{ubaru2017fast}. It could be regarded as an adaptive polynomial approximation technique, where the coefficients of each $\g_i, G \cdot \g_i, \cdots, G^m \cdot \g_i$ are chosen according to the kernel matrix $G$ and random vector $g$. We summarize such approximation in Algorithm \ref{alg:LanczosAlgo}. The Lanczos method is shown to 
achieve faster convergence rate compared to explicit approximations such as Taylor or Chebyshev series, as shown in Theorem \ref{th:LanczosBound}.
\begin{algorithm}
    \caption{Fractional order matrix-based R\'enyi's entropy estimation via Lanczos iteration}
    \label{alg:LanczosAlgo}
    \begin{algorithmic}[1]
    	\STATE \textbf{Input:} Kernel matrix $G \in \mathbb{R}^{n \times n}$, number of random vectors $s$, fractional order $\alpha > 0$, Lanczos steps $m$.
    	\STATE \textbf{Output:}  Approximation to $S_\alpha(G)$.
    	\STATE Generate $s$ i.i.d. random Rademacher vectors $\g_i$, $i = 1, \cdots, s$, i.e. $\mathbb{P}\{(\g_i)_j = 1\} = \mathbb{P}\{(\g_i)_j = -1\} = \frac{1}{2}$.
    	
    	\FOR{$i = 1, 2, \cdots, s$}
    	\STATE Set $\q_0 = 0, \beta_0 = 0, \q_1 = \g_i / \norm{\g_i}$.
    	\FOR{$j = 1, 2, \cdots, m$}
    	\STATE $\hat{\q}_{j+1} = A\q_j - \beta_{j-1}\q_{j-1}$, $\gamma_j = \ang{\hat{\q}_{j+1}, \q_j}$.
    	\STATE $\hat{\q}_{j+1} = \hat{\q}_{j+1} - \gamma_j\q_j$.
    	\STATE Orthogonalize $\hat{\q}_{j+1}$ against $\q_1, \cdots, \q_{j-1}$.
    	\STATE $\beta_j = \norm{\hat{\q}_{j+1}}$, $\q_{j+1} = \hat{\q}_{j+1} / \beta_j$.
    	\ENDFOR
    	\STATE Let $\p$ be the first column of $T^\alpha$, where \\
    	$\qquad T = \begin{vmatrix}
            \gamma_1 & \beta_1 & & \\
            \beta_1 & \gamma_2 & \ddots & \\
            & \ddots & \ddots & \beta_{m-1} \\
            & & \beta_{m-1} & \gamma_m \\
        \end{vmatrix}$.
        \STATE Calculate $F_i = \g_i^T \sum_{k=1}^m (\p)_k \q_k$.
    	\ENDFOR
    	
    	\STATE \textbf{Return:} $\tilde{\S}_\alpha(G) = \frac{1}{1-\alpha} \log_2 \Big(\frac{\sqrt{n}}{s} \sum_{i=1}^s F_i \Big)$.
    \end{algorithmic}
\end{algorithm}
\begin{theorem}
    \label{th:LanczosBound}
    Let $G \in \mathbb{R}^{n \times n}$ be a normalized PSD kernel matrix and $\tilde{\S}_\alpha(G)$ be the output of Algorithm \ref{alg:LanczosAlgo} with $s = \ceil{24\ln(2/\delta)/\epsilon^2}$ and
    \begin{equation*}
    	m = \ceil*{\frac{1}{4} \sqrt{\kappa} \log\prn*{\frac{\kappa^{\alpha+\frac{1}{2}}}{\epsilon}}}.
    \end{equation*}
    Then with confidence at least $1-\delta$, 
    \begin{equation*}
    	\abs*{\S_\alpha(G) - \tilde{\S}_\alpha(G)} \le \abs*{\frac{1}{1-\alpha} \log_2(1 - \epsilon)}.
    \end{equation*}
\end{theorem}
\begin{remark}
   Theorem \ref{th:LanczosBound} shows that $s = \mathcal{O}(1/\epsilon^2)$ and $m = \mathcal{O}\prn*{\max(\alpha,1)\sqrt{\kappa}\log(\kappa/\epsilon)}$ are sufficient to obtain a good approximation. Note that $m$ in Algorithm \ref{alg:LanczosAlgo} is the number of Lanczos steps, which can also be regarded as the degree of the Lanczos method. Compared with  $ m= \mathcal{O}(\kappa\log(1/\kappa\epsilon))$ of Theorem \ref{th:TaylorBound} and $ m=\mathcal{O}(\sqrt[2\alpha]{1/\epsilon})$ of Theorem \ref{th:ChebyBound}, Theorem \ref{th:LanczosBound} only needs $m=\mathcal{O}(\sqrt{\kappa}\log(\kappa/\epsilon))$ to achieve the desired approximation accuracy in consideration of all these approximations require $s = \mathcal{O}(1/\epsilon^2)$. This implies that the Lanczos method enjoys better theoretical properties than explicit polynomial approximation methods. However, it requires extra reorthogonalization in limited-precision computation models, which causes higher computational cost.
\end{remark}

\begin{remark}
    It worth noting that in most of the derived error bounds, the selection of $\epsilon$ is independent of $n$, which means that $s$ and $m$ do not need to scale up along with $n$ to guarantee the same level of accuracy. Therefore, for a large enough $n$, we can achieve arbitrary level of accuracy by selecting increasingly larger $s$ and $m$ under the precondition that $s \ll n$ and $m \ll n$. For some of the other bounds, e.g. the $\lambda_{min} = 0$ case in Theorem 2, 3, the upper bounds involve $n$ so that $m$ needs to scale up with $n$. When $\alpha > 1$, we have $m = O(\sqrt[\alpha]{n/\epsilon})$ in Theorem 2 (or $m = O(\sqrt[2\alpha]{n/\epsilon})$ in Theorem 3) so that the increment of $m$ is slower than $n$ (i.e. $m = o(n)$), which results in the same conclusion as above. Otherwise if $\alpha < 1$ which is rarely the case, there exists an upper limit for the achievable level of approximation accuracy (i.e. a lower bound for $\epsilon$) using the developed algorithms under the precondition that $s \ll n$ and $m \ll n$. However, our simulation studies show that the achieved accuracy is still desirable in this situation (please refer to the appendix).
\end{remark}

\section{Approximating Matrix-based R\'enyi's Entropy: From the Viewpoint of Matrix Approximation}
Previous section establishes the efficient approximations from the viewpoint of randomized linear algebra, in which the main computation cost comes from matrix-vector multiplications required in both power iteration and trace estimation, leading to $\mathcal{O}(ms\cdot \nnz(G))$ time complexity, where $m = \alpha$ for integer $\alpha$-orders and $\nnz(G)$ denotes the cost of calculating matrix-vector multiplication using the straightforward approach. In view of the fact that the kernel matrix is usually dense in practice, $\nnz(G)$ is proportional to $\mathcal{O}(n^2)$, which is still unaffordable in large-scale data science tasks. We hence expect to take full advantage of the structure of the kernel matrix $G$ to further reduce the computation burden. To this end, we first establish the connection between matrix-based R\'enyi's entropy approximation and kernel matrix approximation in the following theorem:
\begin{theorem}
    \label{th:NormBound}
    For any symmetric PSD matrix $G$ and its symmetric approximation $\tilde{G}$, the matrix-based R\'enyi's $\alpha$-order entropy of $\tilde{G}$ is bounded by
    \begin{equation*}
    	\abs*{\S_\alpha(G) \!-\! \S_\alpha(\tilde{G})} \le \abs*{\frac{\alpha}{1 - \alpha} \log_2 (1 - \sqrt{n} \norm{G^{-1}}_2 \norm{G - \tilde{G}}_2)}.
    \end{equation*}
\end{theorem}
\begin{proof}
    Noticing that
    \begin{align}
    	\abs*{\S_\alpha(G) - \S_\alpha(\tilde{G})} &= \abs*{\frac{1}{1-\alpha} \log_2 \frac{\sum_{i=1}^n \lambda_i^\alpha(\tilde{G})}{\sum_{i=1}^n \lambda_i^\alpha(G)}} \nonumber\\
    	&\le \abs*{\frac{1}{1-\alpha} \log_2 \max_i \frac{\lambda_i^\alpha(\tilde{G})}{\lambda_i^\alpha(G)}} \nonumber\\
    	&= \abs*{\frac{\alpha}{1-\alpha} \log_2 \max_i \frac{\lambda_i(\tilde{G})}{\lambda_i(G)}}. \label{MaxEigen}
    \end{align}
    Note that there is no restriction on the order of eigenvalues, i.e. we can reorder the eigenvalues $\lambda_i(G)$ and $\lambda_i(\tilde{G})$ arbitrarily for $i=1,...,n$. To upper bound the eigenvalue ratio above, we introduce the following lemma:
    \begin{lemma} \cite{li2005}
        there exists a permutation $\tau$ of $\{1, \cdots, n\}$ such that for normal non-singular matrix $A$,
        \begin{equation}
        	\max \abs*{\frac{\mu_{\tau(i)} - \lambda_i}{\lambda_i}} \le \sqrt{n(n-s+1)} \norm{A^{-1}}_2 \norm{E}_2.
        	\label{PertBound}
        \end{equation}
        where $\lambda_i$ and $\mu_i$ are eigenvalues of $A$ and $\tilde{A} = A + E$ respectively. $s$ comes from assuming that there exists a unitary matrix $U$ such that $U^*AU = \mathbf{diag}(A_1, ..., A_s)$, where $1 \le s \le n$, and $A_i$ is an upper triangular matrix for $i = 1, ..., s$.
    \end{lemma}
    When the approximation $\tilde{G}$ is symmetric, $\tilde{G}$ is also a normal matrix and $s = n$ is guaranteed. Combining (\ref{MaxEigen}) and (\ref{PertBound}) yields the final result.
\end{proof}

Theorem \ref{th:NormBound} verifies the availability of reducing the computational burden by approximating the kernel matrix. A common solution is to build approximations using low-rank algorithms e.g. the Nystr\"om method, which are shown to be an effective technique to tackle large-scale datasets with no significant decrease in performance \cite{williams2000using}. However, computing matrix-based R\'enyi's entropy involves every eigenvalue of $G$, which is contradictory to the behavior of these algorithms to keep only larger eigenvalues and ignore smaller ones. This results in losing a large portion of the information contained in the eigenspectrum, and can greatly harm the performance for matrix-based R\'enyi's entropy approximation.

Motivated by the recent progress on kernel approximation \cite{sisi2017memory}, we consider exploring both the clustering and low-rank structure of kernel matrices. In case of the most widely used shift-invariant kernels e.g. Gaussian kernel and Laplacian kernel, we can rewrite their functions as $\varphi(\x_i, \x_j ) = g(\sigma(\x_i - \x_j))$, where $g: \mathbb{R}^d \rightarrow \mathbb{R}$ is a measurable function and $\sigma$ is the scale parameter. When $\sigma$ is large enough, the off-diagonal entries in $G$ become small in magnitude compared to diagonal elements. In such cases, the majority of information is stored in the diagonal blocks, so it is natural to adopt low-rank approximation to reduce computational cost. However, for relatively small $\sigma$, the impacts of off-diagonal blocks cannot be ignored since the kernel matrix remains dense. To accommodate a wider range of $\sigma$ values, we employ a block low-rank structure to reduce the computational burden. Take some partition $\mathcal{V}_1, \cdots, \mathcal{V}_c$ of the given samples, where $\mathcal{V}_1 \cup \cdots \cup \mathcal{V}_c = \{n\}$ and $\mathcal{V}_s \cap \mathcal{V}_t = \phi$ for any $1 \le s < t \le c$, the kernel matrix $G$ can be approximated by $\tilde{G}$ with both block and low-rank structure as given below:
\begin{equation} \label{eq:approx}
    G \approx \tilde{G} = \left[
    \begin{matrix}
    	G^{(1,1)}   & G_k^{(1,2)} & \cdots & G_k^{(1,c)} \\
    	G_k^{(2,1)} & G^{(2,2)}   & \cdots & G_k^{(2,c)} \\
    	\vdots      & \vdots      & \ddots & \vdots      \\
    	G_k^{(c,1)} & G_k^{(c,2)} & \cdots & G^{(c,c)}
    \end{matrix}
    \right],
\end{equation} 
where $G_k^{(s,t)}$ denotes the rank-$k$ approximation of the sub-matrix $G^{(s,t)}$ constructed by rows $\mathcal{V}_s$ and columns $\mathcal{V}_t$ of the original kernel matrix $G$. Observing that
\begin{equation*}
    \|G - \tilde{G}\|_F^2 \leq \sum_{i,j} \varphi(\x_i, \x_j)^2 - \sum_{s=1}^c \sum_{i,j \in \mathcal{V}_s} \varphi(\x_i, \x_j)^2,
\end{equation*}
therefore, minimizing the difference between $G$ and $\tilde{G}$ is equivalent to maximizing the second term:
\begin{equation}
    \min \|G - \tilde{G}\|_F^2 \Leftrightarrow \max \sum_{s=1}^c \sum_{i,j \in \mathcal{V}_s} \varphi(\x_i, \x_j)^2 := D^{kernel}.
\end{equation}
However, directly maximizing $D^{kernel}$ can result in all the data being assigned to the same cluster. A common way to solve this problem is to normalize $D$ by each cluster's size $|\mathcal{V}_s|$. This leads to the spectral clustering objective:
\begin{equation} \label{eq:spectral_cluster}
    D^{kernel}(\{\mathcal{V}_s\}_{s=1}^c) = \sum_{s=1}^c \frac{1}{\abs{\mathcal{V}_s}} \sum_{i,j \in \mathcal{V}_s} \varphi(\x_i, \x_j)^2.
\end{equation}
Directly optimizing (\ref{eq:spectral_cluster}) is computational infeasible since we have to iterate all possibilities of the partition $\{\mathcal{V}_s\}_{s=1}^c$. For applicable minimization in practice, we adopt a lower bound for our objective $D^{kernel}(\{\mathcal{V}_s\}_{s=1}^c)$:
\begin{theorem}
    For any shift-invariant Lipschitz continuous kernel function $\varphi$,
    \begin{equation*}
    	D^{kernel}(\{\mathcal{V}_s\}_{s=1}^c) \ge \frac{1}{2n} - R^2 D^{kmeans}(\{\mathcal{V}_s\}_{s=1}^c),
    \end{equation*}
    where $R$ is a constant depending on the kernel function, and
    \begin{equation*}
    	D^{kmeans}(\{\mathcal{V}_s\}_{s=1}^c) \equiv \sum_{s=1}^c \frac{1}{\abs{\mathcal{V}_s}} \sum_{i,j \in \mathcal{V}_s} \norm{x_i - x_j}_2^2
    \end{equation*}
    is the $k$-means objective function.
    \label{th:ObjBound}
\end{theorem}
\begin{proof}
    For any shift-invariant kernel function, there exists a real valued function $f \in \mathbb{R}^d \mapsto \mathbb{R}$ so that for any data points $\bm{x}_i, \bm{x}_j \in \mathbb{R}^d$,
    \begin{equation*}
    	\varphi(\bm{x}_i, \bm{x}_j) = f(\bm{x}_i - \bm{x}_j) \ge f(\mathbf{0}) - L\norm{\bm{x}_i - \bm{x}_j}_2,
    \end{equation*}
    where $L$ is the Lipschitz constant of $f(\cdot)$. Then the elements of matrix $G$ satisfies:
    \begin{equation*}
    	\begin{split}
    		G_{ij} &= \frac{1}{n}\frac{\kappa(\bm{x}_i, \bm{x}_j)}{\sqrt{\kappa(\bm{x}_i, \bm{x}_i)\kappa(\bm{x}_j, \bm{x}_j)}} = \frac{f(\bm{x}_i - \bm{x}_j)}{n f(\bm{0})} \\
    		&\ge \frac{1}{n f(\bm{0})}\prn*{f(\mathbf{0}) - L\norm{\bm{x}_i - \bm{x}_j}_2} \\
    		&= \frac{1}{n} - \frac{L}{n f(\bm{0})}\norm{\bm{x}_i - \bm{x}_j}_2.
    	\end{split}
    \end{equation*}
    Let $R \equiv L/n f(\bm{0})$, we have
    \begin{equation*}
    	G_{ij} + R\norm{\bm{x}_i - \bm{x}_j}_2 \ge \frac{1}{n}.
    \end{equation*}
    Taking the square of both sides
    \begin{equation*}
    	G_{ij}^2 + R^2\norm{\bm{x}_i - \bm{x}_j}_2^2 + 2 G_{ij}R\norm{\bm{x}_i - \bm{x}_j}_2 \ge \frac{1}{n^2}.
    \end{equation*}
    From the classical arithmetic and geometric mean inequality, we get the following bound:
    \begin{equation*}
    	2 G_{ij}R\norm{\bm{x}_i - \bm{x}_j}_2 \le G_{ij}^2 + R^2\norm{\bm{x}_i - \bm{x}_j}_2^2,
    \end{equation*}
    and therefore:
    \begin{equation*}
    	G_{ij}^2 + R^2\norm{\bm{x}_i - \bm{x}_j}_2^2 \ge \frac{1}{2n^2}.
    \end{equation*}
    According to the definition of $D^{kernel}$:
    \begin{equation*}
    	\begin{split}
    		D^{kernel}(\{\mathcal{V}_s\}_{s=1}^c) &= \sum_{s=1}^c \frac{1}{\abs{\mathcal{V}_s}} \sum_{i, j \in \mathcal{V}_s} G_{ij}^2 \\
    		&\ge \sum_{s=1}^c \frac{1}{\abs{\mathcal{V}_s}} \sum_{i, j \in \mathcal{V}_s} \prn*{\frac{1}{2n^2} - R^2 \norm{\bm{x}_i - \bm{x}_j}_2^2} \\
    		&= \frac{1}{2n} - R^2 \sum_{s=1}^c \frac{1}{\abs{\mathcal{V}_s}} \sum_{i, j \in V_s} \norm{\bm{x}_i - \bm{x}_j}_2^2 \\
    		&= \frac{1}{2n} - R^2 D^{kmeans}(\{\mathcal{V}_s\}_{s=1}^c),
    	\end{split}
    \end{equation*}
    which finishes the proof.
\end{proof}

Theorem \ref{th:ObjBound} suggests that the $k$-means algorithm is an ideal choice for selecting the partition $\mathcal{V}_1, \cdots, \mathcal{V}_c$. The progress of block low-rank kernel approximation is summarized in Algorithm \ref{alg:LowRankApprox}, where the $k$ largest singular values and the corresponding singular vectors are calculated by Randomized SVD algorithm \cite{halko2011} in practice, which takes $\mathcal{O}(n_r n_ck)$ arithmetic operations for a $n_r \times n_c$ sub-matrix. Through the block low-rank approximation, the complexity of matrix-vector multiplication is reduced to $\mathcal{O}(n^2/c+nck)$. The following proposition gives the error bound of kernel approximation using the given strategy:

\begin{algorithm}
    \caption{Block Low-Rank Kernel Matrix Approximation}
    \label{alg:LowRankApprox}
    \begin{algorithmic}[1]
    	\STATE \textbf{Input:} Shift-invariant kernel matrix $G \in \mathbb{R}^{n \times n}$, rank $k$, number of clusters $c$.
    	\STATE \textbf{Output:}  Approximation to $G$.
    	\STATE Obtain a partition $\mathcal{V}_1, ..., \mathcal{V}_c$ by the $k$-means algorithm with $c$ clusters.
    	\STATE Rearrange the kernel matrix $G$ as a block matrix by the partition $\mathcal{V}_1, \cdots, \mathcal{V}_c$. 
    	\STATE Obtain the best rank-$k$ approximation of each off-diagonal block.
    	
    	\STATE \textbf{Return:}  Construct $\tilde{G}$ by (\ref{eq:approx}).
    \end{algorithmic}
\end{algorithm}
\begin{proposition}
    Given samples $\{ \x_i\}_{i=1}^n \subset \mathbb{R}^d$ with a partition $\mathcal{V}_1, ..., \mathcal{V}_c$. Let $\tilde{G}$ be the output of Algorithm \ref{alg:LowRankApprox} and let the radius of partition $\mathcal{V}_i$ be $r_i$ for $ 1 \leq i \leq c $. Assuming $r_1 \le r_2 \le ... \le r_c$, then for any shift-invariant Lipschitz continuous kernel function, we have
    \begin{equation}
    	\norm{G - \tilde{G}}_F \le 4Lk^{-\frac{1}{d}} \sqrt{2r},
    \end{equation}
    where $L$ is the Lipschitz constant of the kernel function, and $r = \sum_{i=1}^c r_i^2 \abs{\mathcal{V}_i} \sum_{j=i+1}^c \abs{\mathcal{V}_j}$.
    \label{MEKA}
\end{proposition}
\begin{proof}
    Denote $G^{(s,t)}$ as one block of matrix $G$ with rows $\mathcal{V}_s$ and columns $\mathcal{V}_t$. By Theorem 2 in \cite{sisi2017memory}, for any shift-invariant Lipschitz continuous kernel function with Lipschitz constant $L$, we have the following bound:
    \begin{equation*}
    	\norm{G^{(s,t)} - G_k^{(s,t)}}_F \le 4Lk^{-\frac{1}{d}} \sqrt{\abs{\mathcal{V}_s}\abs{\mathcal{V}_t}} \min(r_s, r_t).
    \end{equation*}
    Summing up the inequality above for all off-diagonal blocks yields the final result:
    \begin{equation*}
    	\begin{split}
    		\norm{G - \tilde{G}}_F &= \sqrt{\sum_{s,t=1...c, s < t} 2\norm{G^{(s,t)} - G_k^{(s,t)}}_F^2} \\
    		&\le 4Lk^{-\frac{1}{d}} \sqrt{\sum_{s,t=1...c, s < t} 2\abs{\mathcal{V}_s}\abs{\mathcal{V}_t}{\min}^2(r_s, r_t)}\\
    		&= 4Lk^{-\frac{1}{d}} \sqrt{2\sum_{i=1}^c r_i^2 \abs{\mathcal{V}_i} \sum_{j=i+1}^c \abs{\mathcal{V}_j}}.
    	\end{split}
    \end{equation*}
\end{proof}
Combining Theorem \ref{th:NormBound} and Proposition \ref{MEKA} together yields the following corollary on approximation error of matrix-based R\'enyi's entropy with block low-rank approximation technique:
\begin{corollary}
    Under the same conditions in Proposition \ref{MEKA}, the approximation error of $\S_\alpha(G)$ is bounded by
    \begin{equation*}
    	\abs*{\S_\alpha(G) - \S_\alpha(\tilde{G})} \leq \abs*{\frac{\alpha}{1 - \alpha} \log_2 (1 - 4\sqrt{2rn} Lk^{-\frac{1}{d}} \norm{G^{-1}}_2)}.
    \end{equation*}
\end{corollary}
\begin{remark}
	Obviously, there is a trade-off in the choice of $k$ and $c$ between time complexity and approximation error. On the one hand, small $k$ or large $c$ may cause higher approximation error, or even worse that the approximated kernel matrix loses its positive semi-definite property. On the other hand, kernel approximation loses its running time improvement for large $k$ or small $c$. Empirically, we suggest $c\approx\sqrt[4]{n}$ and $k\approx\sqrt{n}$ for the best performance.
\end{remark}

\subsection{Approximation for Matrix-based R\'enyi's Mutual Information and Total Correlation}
In the discussion above, we developed randomized approximations through both stochastic trace estimators and block low-rank kernel approximations for matrix-based R\'enyi's entropy. As natural extensions of the R\'enyi's $\alpha$-order entropy functional, matrix-based R\'enyi's mutual information (Eq.~(\ref{eq:Renyi_MI})) and total correlation (Eq.~(\ref{eq:Renyi_TC})) can also be directly approximated by the established algorithms, since they are just linear combinations of the original definition. Moreover, mutual information and total correlation are fundamental information statistics for most downstream machine learning and neuroscience tasks, such as measuring dependence or causality, feature selection and regularizations in deep neural training. Comprehensive experiments in Section \ref{sec:real_world} demonstrate the excellent performance of the proposed approximation algorithms in multiple real-world data science tasks.

\section{Experimental Results} \label{sec:exp}
In this section, we evaluate the performance of the proposed approximation algorithms on large-scale simulation and real-world datasets. All numerical studies are conducted with an Intel i7-10700 (2.90GHz) CPU, an RTX 2080Ti GPU and 64GB of RAM. Approximation algorithms are implemented in C++ and Python, with k-means algorithm provided by OpenCV \cite{opencv_library}, fundamental linear algebra functions provided by Eigen \cite{eigenweb} and Pytorch respectively.

\subsection{Simulation Studies}
In our experiments, the simulation data are generated by mixture of Gaussian distribution $\frac{1}{2}N(-1,\I_d)+\frac{1}{2}N(1,\I_d)$ with $n = 10,000$ and $d = 10$, the size of the kernel matrix is then $10,000 \times 10,000$. We choose Gaussian kernel $\varphi(\x_i, \x_j) = \exp(-\| \x_i - \x_j \|_2^2/2\sigma^2)$ with $\sigma = 1$ in kernel matrix construction. We set order of rank $k = 80$ and number of clusters $c=20$ in block low-rank approximation experiments. To reduce the impact of randomness, we run each test for $K = 100$ times and report the mean relative error (MRE) and corresponding standard deviation (SD). The oracle $\S_\alpha(G)$ is computed through the trivial $\mathcal{O}(n^3)$ eigenvalue approach. We use the high resolution clock provided by the C++ standard library, whose precision reaches nanoseconds. For comparison, the straightforward eigenvalue approach takes $219.8$ seconds for a $10000 \times 10000$ kernel matrix.

\subsubsection{Randomized Approximation}
We examine the performance of the trace estimation algorithm ``Trace'' and its combination with block low-rank approximation ``Trace + Low-rank'' for R\'enyi's entropy with integer orders. We choose $\alpha = \{2, 3, 5, 8\}$ in Algorithm \ref{alg:IntAlgo} and the number of random Gaussian vectors $s$ ranges from $20$ to $200$ in increments of $20$. The Time vs. MRE curves for different $\alpha$ are shown in Figure \ref{IntExp}, in which the shaded area indicates $\pm 0.25$ SD. It is easy to observe that when $s$ is large enough, ``Trace + Low-rank"  runs faster than ``Trace" while yielding comparable estimation error.

\begin{figure}[t]
    \centering
    \includegraphics[width=0.45\textwidth]{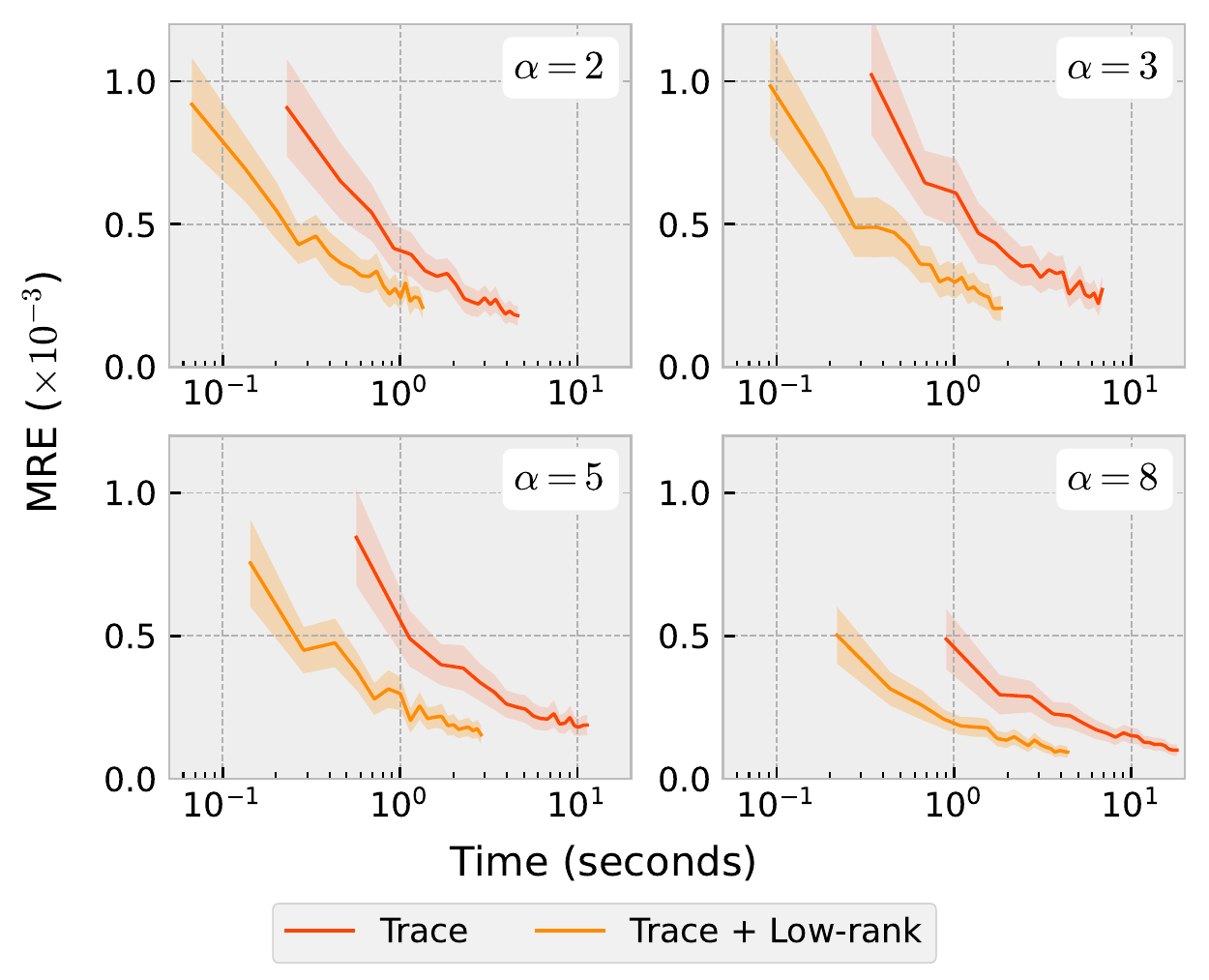}
    \caption{Time vs. MRE curves for integer $\alpha$-order R\'enyi's entropy estimation.}
    \label{IntExp}
\end{figure}

\subsubsection{Fractional Order Approximation}
We further evaluate the approximation effect for fractional $\alpha$ orders. The impact of $\alpha$ on MRE for Taylor, Chebyshev and Lanczos polynomial approximations are reported in Figure \ref{AlphaExp}. It can be seen that the approximation error grows with the increase of $\alpha$ for $\alpha <1$ and decreases otherwise, which supports our claims in Theorems \ref{th:TaylorBound}, \ref{th:ChebyBound} and \ref{th:LanczosBound} in which the term $|\frac{1}{1-\alpha}|$ dominates the approximation error for fractional $\alpha$-orders.

\begin{figure}[t]
    \centering
    \includegraphics[width=0.45\textwidth]{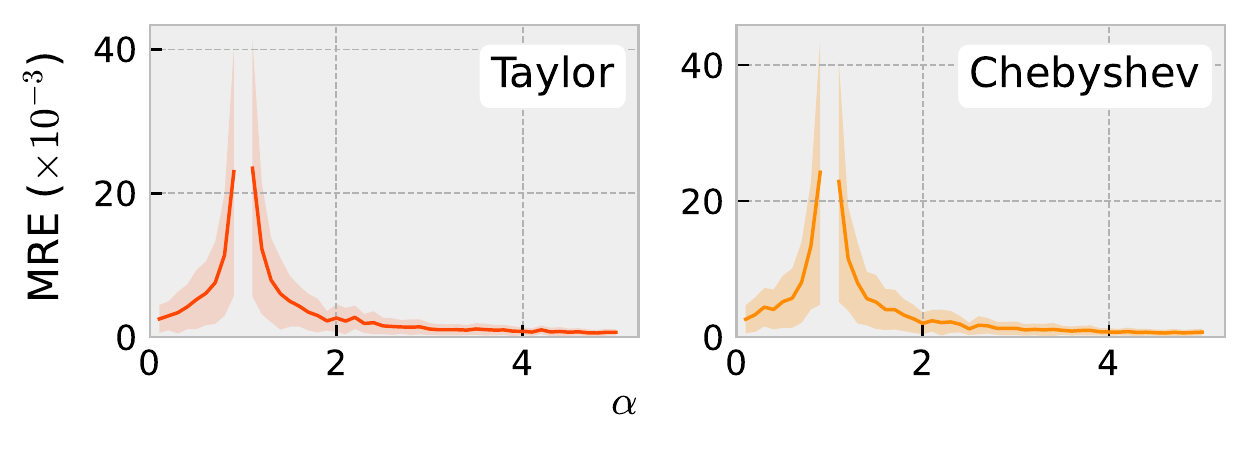}
    \includegraphics[width=0.25\textwidth]{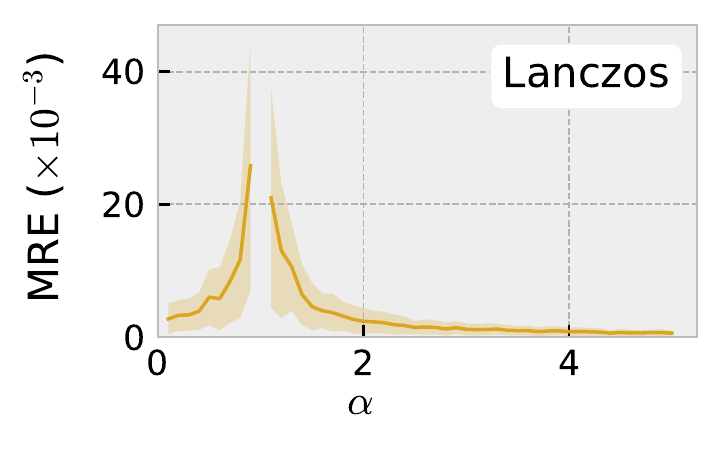}
    \caption{$\alpha$ vs. MRE curves for fractional $\alpha$-order algorithms.}
    \label{AlphaExp}
\end{figure}

Next, we test the approximation precision of Algorithms \ref{alg:Taylor}, \ref{alg:ChebyAlgo} and  with or without block low-rank kernel approximation with polynomial degree $m=30$. For Algorithm \ref{alg:LanczosAlgo}, we use the naive implementation where the Lanczos vectors are reorthogonalized every iteration with $m=15$, since it achieves a faster convergence rate. Again, the number of random Gaussian vectors $s$ ranges from $20$ to $200$ in increments of $20$, and the results are shown in Figure \ref{NonIntExp}. It can be observed that the Lanczos method runs slower than Taylor or Chebyshev. For a wide range of different $\alpha$ values, all of the three methods produce similar MRE. Besides, the block low-rank methods save half of the running time while only introducing negligible error in the approximation.

\begin{figure*}[!t]
    \centering
    \includegraphics[width= 0.8\textwidth]{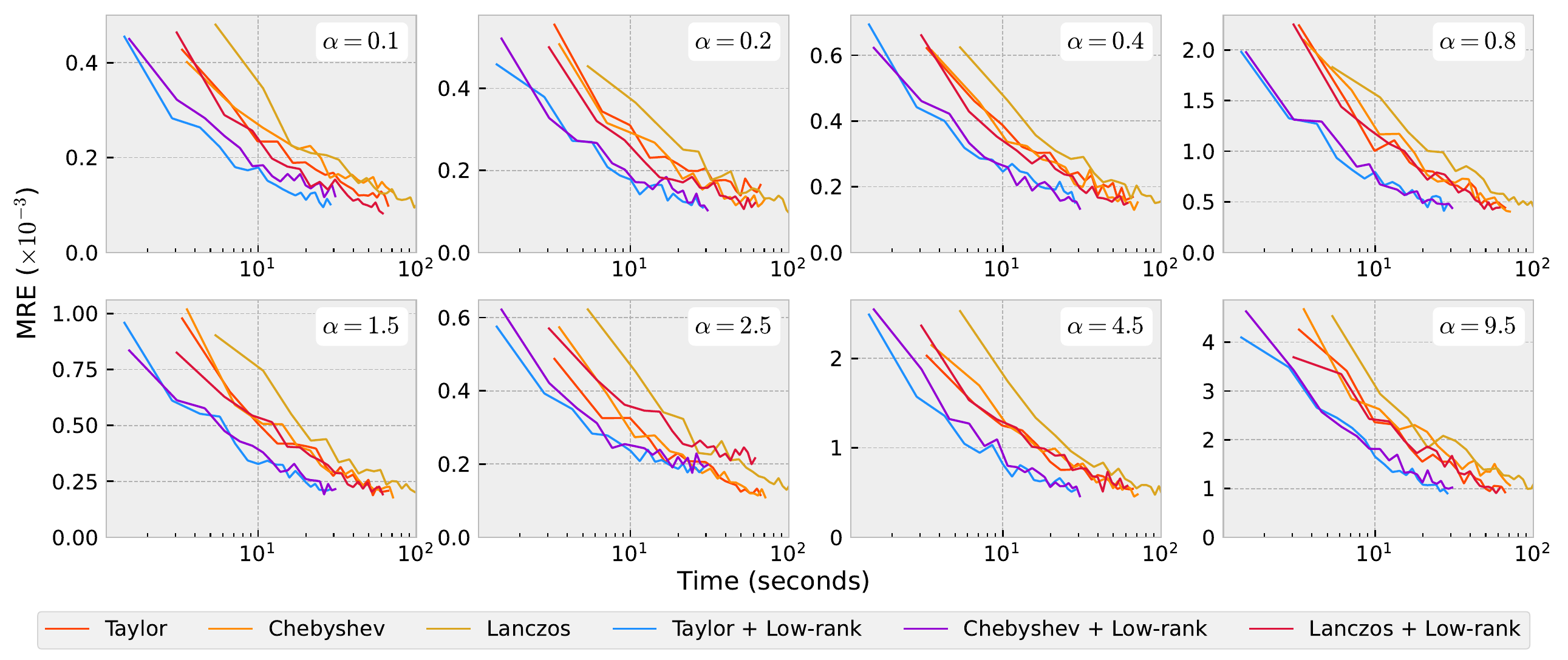}
    \caption{Time vs. MRE curves for fractional $\alpha$-order R\'enyi's entropy estimation.}
    \label{NonIntExp}
\end{figure*}

Next, we demonstrate that the proposed algorithms are applicable to different kinds of kernel functions. Considering the prerequisite that matrix-based R\'enyi’s entropy supports only infinitely divisible kernels, we choose the following widely-used ones to constitute our benchmark:
\begin{itemize}
    \item Polynomial kernel: $\varphi(\x_i, \x_j) = (\x_i^\top \x_j + r)^p$ with $r=1$ and $p=2, 4$.
    \item Gaussian kernel: $\varphi(\x_i, \x_j) = \exp(-\| \x_i - \x_j \|_2^2/2\sigma^2)$ with $\sigma=0.5, 1$.
\end{itemize}
We set $\alpha = 1.5$ and $m = 50$, while $s$ ranges from $10$ to $100$ in increments of $10$. We keep the previous settings and for Gaussian kernel, we add further comparisons with block low-rank approximation due to its shift-invariant property. The evaluation results are presented in Figure \ref{KernelExp}. It can be seen that all methods achieve consistent approximations under different kernel settings, which proves that our algorithm can accommodate different kernel functions.

\begin{figure}[t]
    \centering
    \includegraphics[width=0.45\textwidth]{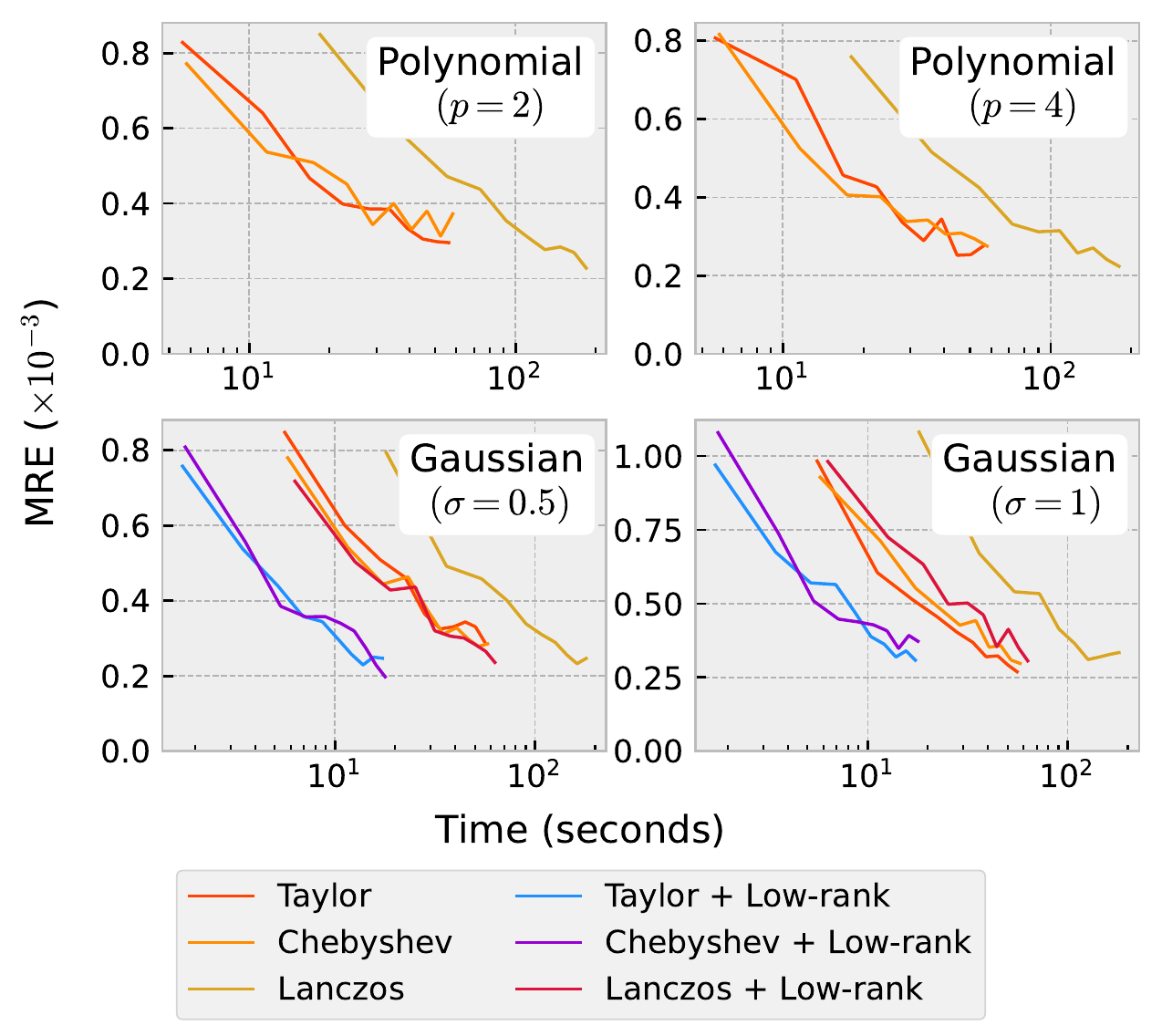}
    \caption{Time vs. MRE curves evaluated on different kernels.}
    \label{KernelExp}
\end{figure}

\subsubsection{Effect of Block Low-rank Approximation}
Lastly, we examine the effect of block low-rank kernel approximation with different number of clusters $c$ and different ranks $k$. We use the Gaussian kernel same as above with $\alpha = 2.5$, $s = 100$ and $m = 40$. To evaluate the impact of $c$ and $k$ respectively, we take a grid search for $c$ varying from $2$ to $20$ with step $2$ and $k$ varying from $10$ to $100$ with step $10$. The results are shown in Figure \ref{ClusterExp} using Chebyshev approximation. It can be seen that with reasonably high $c$ and $k$ values, block low-rank approximation achieves nearly the same accuracy as the original method, while saving nearly half of running time. This phenomenon demonstrates that our approximations can reach a satisfactory trade-off in a suitable range of both $c$ and $k$. The result also provides some insights on how to appropriately select the most suitable approximation in real applications: the running time is decreasing with $c$ at first along with the decrease of MRE, which is due to the reduced complexity $\mathcal{O}(n^2/c+nck)$ of matrix-vector multiplication. In this sense, we recommend a range of $c$ between $10$ and $20$ and a range of $k$ between $50$ and $100$ to balance running time and approximation precision.

Finally, we evaluate such a hyper-parameter combination for a $50000 \times 50000$ Gram matrix. ``Chebyshev + low-rank'' approximations take less than $8$ minutes to obtain the entropy value, whereas the original estimator requires more than $8$ hours. This further validates the remarkable performance of our approximation algorithms.

\begin{figure}[t]
    \centering
    \includegraphics[width=0.5\textwidth]{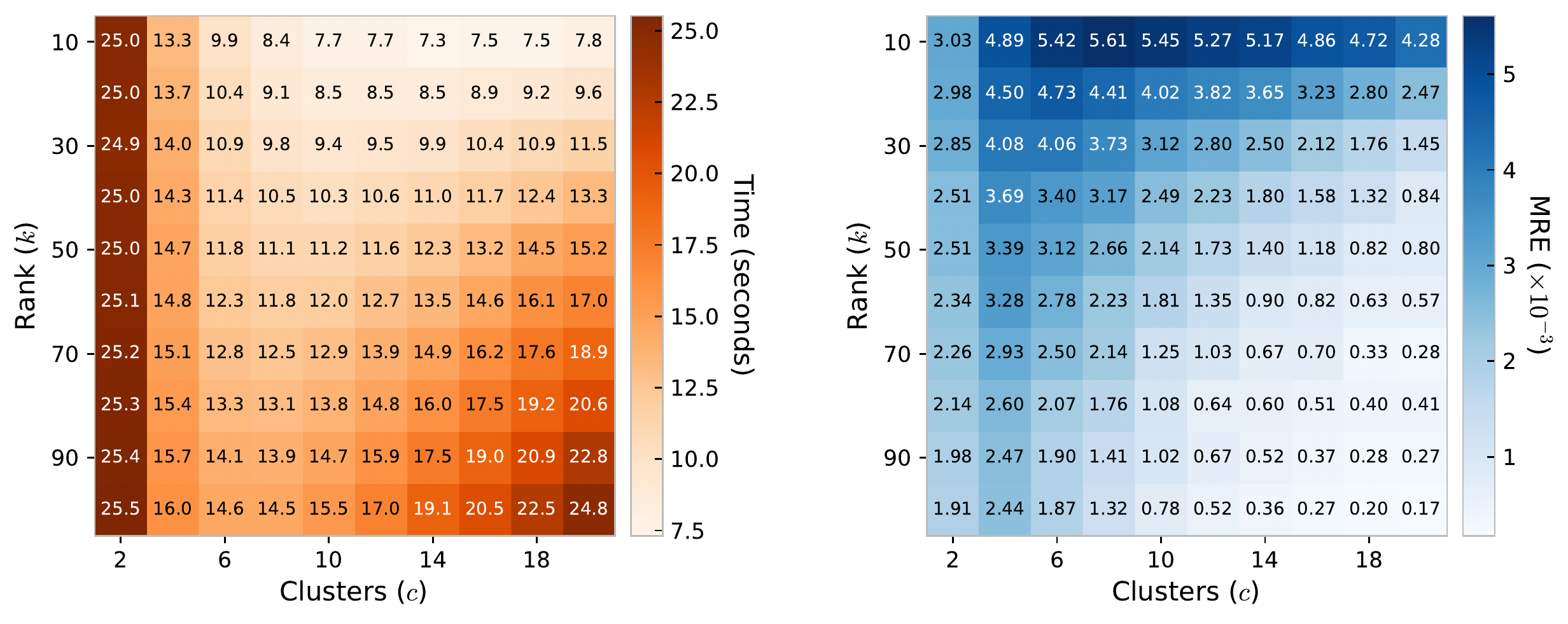}
    \caption{Impact of $c$ and $k$ on Time and MRE in block low-rank kernel approximation. For comparison, the Chebyshev algorithm without low-rank approximation yields MRE $0.19 \times 10^{-3}$ in $46.9$ seconds.}
    \label{ClusterExp}
\end{figure}

\subsection{Real-world Data Studies}\label{sec:real_world}
Following the Venn diagram relation for Shannon entropies~\cite{yeung1991new}, the matrix-based R\'enyi's $\alpha$-order mutual information and total correlation are defined accordingly as shown Eqs.~(\ref{eq:Renyi_MI}) and (\ref{eq:Renyi_TC}). In real-world scenarios, these extended information quantities have much more widespread applications including feature selection, dimension reduction and information-based clustering. In this section, we apply our algorithms on three representative information-based learning tasks, namely information bottleneck, multi-view information bottleneck and feature selection, to demonstrate the acceleration effect of both matrix-based R\'enyi's entropy and mutual information. We select Gaussian kernel with $\sigma=1$ and $\alpha=2$ in the following experiments for simplicity.

\subsubsection{Information Bottleneck for Robust Deep Learning}
The Information Bottleneck (IB) objective was firstly introduced by \cite{tishby1999information} and has recently achieved great success in preventing overfitting in deep network training \cite{yu2021deep, ardizzone2020training, wu2020graph}. IB aims to learn a representation $\T$ by maximizing $\I(\Y, \T)$ and minimizing $\I(\X, \T)$ simultaneously, forcing the network to ignore irrelevant information of $\X$ about $\Y$ and thus improving both robustness and generalization. IB is formulated as:
\begin{equation}\label{eq:IB_Lagrangian}
    \mathcal{L}_{IB} = \I(\Y, \T) - \beta\I(\X, \T).
\end{equation}

In deep neural networks, $\X$ denotes the input variable, $\Y$ denotes the desired output (e.g., class labels) and $\T$ refers to the latent representation of one hidden layer. Usually, this can be done by optimizing the IB Lagrangian (Eq.~(\ref{eq:IB_Lagrangian})) via a classic cross-entropy (CE) loss regularized by a differentiable mutual information term $\I(\X;\T)$. In practice, $\I(\X;\T)$ can be measured by variational approximations \cite{alemi2016deep, kolchinsky2019nonlinear}, mutual information neural estimator (MINE) \cite{belghazi2018mutual, elad2019direct} or the matrix-based R\'enyi's entropy \cite{yu2021deep}. 

\begin{table*}[t]
    \centering
    \caption{Classification accuracy (\%) against adversarial attack, and time spent on calculating IB objective (left) and training network (right) for different methods on CIFAR-10 (the second-best performances are underlined).}
    \label{IBRes}
    \begin{tabular}{ l||ccccccc|c }
    	\hline
    	\multirow{2}{*}{Method} & \multicolumn{7}{c|}{Perturbation Amount ($\epsilon$)} & \multirow{2}{*}{Time (h)} \\
    	\cline{2-8}
    	& 0.00 & 0.05 & 0.10 & 0.15 & 0.20 & 0.25 & 0.30 & \\
    	\hline
    	Cross Entropy & 93.10 & 63.19 & 55.82 & 51.51 & 45.90 & 41.60 & 38.61 & \quad - \enspace / 2.27 \\
    	VIB \cite{alemi2016deep} & 93.77 & 65.33 & 58.14 & 53.02 & 47.93 & 44.01 & 40.37 & 0.03 / 2.30 \\
    	NIB \cite{kolchinsky2019nonlinear} & 94.07 & 68.45 & 61.51 & 56.32 & 52.71 & 46.09 & 39.85 & 0.11 / 2.38 \\
    	RIB \cite{yu2021deep} ($\alpha=1.01$) & \textbf{94.36} & \textbf{69.68} & \textbf{62.45} & \textbf{57.76} & \underline{52.92} & \underline{47.93} & \underline{41.78} & \multirow{2}{*}{1.13 / 3.40} \\
    	RIB \cite{yu2021deep} ($\alpha=2$) & 94.28 & \underline{69.48} & \underline{61.64} & 56.89 & 52.79 & 47.56 & 41.19 & \\
    	\hline
    	ARIB ($\alpha=2$) & \underline{94.30} & 69.33 & 61.47 & \underline{57.08} & \textbf{52.97} & \textbf{48.41} & \textbf{43.19} & 0.23 / 2.50 \\
    	\hline
    \end{tabular}
\end{table*}

Following the experiment settings in \cite{alemi2016deep, yu2021deep}, we select VGG16 \cite{simonyan2015very} as the baseline network and CIFAR-10 as the classification dataset. We choose the last fully connected layer ahead softmax as the bottleneck layer. All models are trained for $400$ epochs with $100$ batch size where the learning rate is $0.1$ initially and reduced by a factor of $10$ every $100$ epochs. We compare the performance of different IB methods, including Variational IB (VIB) \cite{alemi2016deep}, Nonlinear IB (NIB) \cite{kolchinsky2019nonlinear}, matrix-based R\'enyi's IB (RIB) \cite{yu2021deep} and our Approximated RIB (ARIB). We follow the suggestion of \cite{yu2021deep} and set $\beta = 0.01$. For ARIB, we set number of random vectors $s=20$. To test the robustness of different methods against adversarial attacks, we adopt the Fast Gradient Sign Attack (FGSM) criterion, in which the adversarial examples are generated by:
\begin{equation*}
    \hat{x} = x + \epsilon \cdot sign(\nabla_x \mathcal{L}(\theta, x, y)),
\end{equation*}
where $x$ denotes the original input, $\hat{x}$ denotes the attack input, $\mathcal{L}(\theta, x, y)$ is the loss with respect to the input $x$ and $\epsilon$ is the perturbation amount. The final classification accuracy and time spent on network training are reported in Table \ref{IBRes}. In \cite{yu2021deep}, the authors choose $\alpha = 1.01$ to recover Shannon's entropy, while the performance with $\alpha = 2$ was not tested. For completeness, we compare the performance of RIB with $\alpha=1.01$ and $\alpha = 2$, which shows that $\alpha = 2$ is also capable of high performance. The final results demonstrate that our ARIB achieves remarkable speedup without sacrificing prediction ability compared to directly applying RIB. In particular, our approximation methods lead to $5$ times accelerating effect, making the training time comparable to VIB or NIB.

\subsubsection{Information Bottleneck for Multi-view Image Classification}
Recently, the IB principle has been extended to multi-view learning tasks to minimize redundant information between multiple input views~\cite{wang2019deep,aguerri2019distributed}. Given different views of features $\X_1, \cdots, \X_k$ and the target label $\Y$, a typical multi-view deep network contains $k$ separate encoders, each maps one specific view of input data to a latent representation $\T_i$, $i \in [1, k]$. These intermediate representations are then fused to a joint one $\Z$ by the classifier to acquire the final classification result. To this end, the multi-view IB objective can be defined as~\cite{wang2019deep}:
\begin{equation*}
    \mathcal{L}_{MIB} = \I(\Y;\Z) - \sum_{i=1}^k \beta_i \I(\X_i;\Z_i),
\end{equation*}
where $\beta_i$, $i \in [1,k]$ balances the trade-off between representability and robustness in the $i$-th view.

In this experiment, we follow the settings in \cite{zhang2022multi} and choose MNIST as our benchmark, where the $k = 2$ views are constructed by randomly rotating input images with an angle in $[\pi/4, \pi/4]$, and adding $[-0.1,0.1]$ uniformly distributed background noise after normalization. The classification network is built upon an MLP of form $512$-$512$-$512$ for each view, $512$-$256$-$10$ for classifier and ReLU activation layers. All models are trained for $60$ epochs with $100$ batch size. We compare the performance of matrix-based R\'enyi's multi-view IB (MEIB) \cite{zhang2022multi}, our approximated MEIB (AMEIB) and Deep Multi-view IB (Deep IB) \cite{wang2019deep}. We select the values of $\beta_1 = 0.001$ and $\beta_2 = 0.01$ through cross-validation. The final results of classification accuracy and training time are reported in Table \ref{MIBRes}. Again, it can be seen that with our approximation methods, we achieved a $3$ times speedup, which is in the same level of training time compared to Deep IB, while bringing negligible drop in classification accuracy. This improvement could be further enlarged with larger batch sizes, which is a common choice in modern fine-tuning techniques \cite{samuel2018dont}.

\begin{table}[t]
    \centering
    \caption{Classification accuracy for multi-view image classification, and time spent on calculating IB objective (left) and training network (right) for different methods on MNIST (the second-best performances are underlined).}
    \label{MIBRes}
    \begin{tabular}{ l||cc }
    	\hline
    	Method & Accuracy (\%) & Time (h) \\
    	\hline
    	Cross Entropy & 96.93 $\pm$ 0.10 & \quad - \enspace / 1.22 \\
    	Deep IB \cite{wang2019deep} & 97.22 $\pm$ 0.14 & 0.06 / 1.28 \\
    	MEIB \cite{zhang2022multi} ($\alpha=1.01$) & \textbf{97.50 $\pm$ 0.08} & \multirow{2}{*}{0.21 / 1.43} \\
    	MEIB \cite{zhang2022multi} ($\alpha=2$) & 97.46 $\pm$ 0.07 & \\
    	\hline
    	AMEIB ($\alpha=2$) & \underline{97.47 $\pm$ 0.13} & 0.07 / 1.29 \\
    	\hline
    \end{tabular}
\end{table}

\subsubsection{Mutual Information for Feature Selection}
Given a set of features $S = \{\X_1, \cdots, \X_n\}$, the feature selection task aims to find the smallest subset of $S$ while trying to maximize its relevance about the labels $\Y$. In the view of information theory, the goal is to maximize $\I(\{\X_{i_1}, \cdots, \X_{i_k}\};\Y)$, where $i_1, \cdots, i_k$ indicate the selected features. However, the curse of high dimension causes this objective extremely hard to estimate. By employing the extension of multivariate matrix-based R\'enyi's mutual information \cite{yu2019multivariate}, we are finally granted direct measure to this information quantity. We evaluate the performance of matrix-based R\'enyi's mutual information (RMI) and RMI approximated by our algorithms (ARMI) on 8 well-known classification datasets used in previous works \cite{uci, koklu2020multiclass, yu2019multivariate, vinh2016can} which cover a wide range of instance-feature ratio, number of classes and data source domains, as shown in Table \ref{FeatSelData}. For completeness, we compare with state-of-the-art information based feature selection methods, including Mutual Information-based Feature Selection (MIFS) \cite{battiti1994using}, First-Order Utility (FOU) \cite{brown2009new}, Mutual Information Maximization (MIM) \cite{lewis1992feature}, Maximum-Relevance Minimum-Redundancy (MRMR) \cite{peng2005feature}, Joint Mutual Information (JMI) \cite{yang1999data} and Conditional Mutual Information Maximization (CMIM) \cite{fleuret2004fast}.

\begin{table}[t]
    \centering
    \setlength{\tabcolsep}{0.8mm}
    \caption{Datasets used in feature selection tasks, and the total running time of RMI and ARMI.}
    \label{FeatSelData}
    \begin{tabular}{ l||cccc|cc }
    	\hline
    	\multirow{2}{*}{Dataset} & \multirow{2}{*}{\# Instance} & \multirow{2}{*}{\# Feature} & \multirow{2}{*}{\# Class} & \multirow{2}{*}{Type} & \multicolumn{2}{c}{Running Time (h)} \\
    	\cline{6-7}
    	&&&&& RMI & ARMI \\
    	\hline
    	Covid & 5434 & 21 & 2 & discrete & 2.56 & 0.47 \\
    	Optdigits & 5620 & 65 & 10 & discrete & 10.68 & 1.25 \\
    	Statlog & 6435 & 37 & 6 & discrete & 8.48 & 1.71 \\
    	Gesture & 11674 & 65 & 4 & discrete & 86.03 & 8.19 \\
    	Spambase & 4601 & 57 & 2 & continuous & 4.88 & 1.15 \\
    	Waveform & 5000 & 40 & 3 & continuous & 4.43 & 0.77 \\
    	Galaxy & 9150 & 16 & 2 & continuous & 8.64 & 1.29 \\
    	Beans & 13611 & 17 & 7 & continuous & 30.35 & 3.11 \\
    	\hline
    \end{tabular}
\end{table}

In our experiments, continuous features are discretized into $5$ bins by equal-width strategy \cite{vinh2014reconsidering} for non-R\'enyi methods. We choose the Support Vector Machine algorithm implemented by libSVM \cite{chang2011libsvm} with RBF kernel ($\sigma=1$) as the classifier and employ 10-fold cross-validation. In our observation, classification error tends to stop decreasing after incrementally selecting the first $k=8$ features, so we select the first $8$ most informative features step by step through a greedy criterion and report the best validation error achieved by each method. Experiment results are shown in Table \ref{FeatSel} in terms of classification error, with corresponding running time reported in Table \ref{FeatSelData} for RMI and ARMI. As we can see, ARMI achieves $5$ to $10$ times speedup compared to original RMI on all datasets while introducing only less than $1\%$ error in classification experiments, which is still significantly lower than other methods in comparison. This demonstrates the great potential of the proposed approximating algorithms in machine learning applications.  It is worth noting that the Galaxy dataset contains an extremely informative feature named "redshift", which leads to higher than $95\%$ classification accuracy solely in identifying the class of a given star. Both RMI and ARMI can pick it out in the first place, but other methods even fail to choose it as the second high-weighted candidate.

\begin{table*}[!t]	
    \centering
    \caption{Feature selection results in terms of classification error (\%).}
    \label{FeatSel}
    \begin{tabular}{ l||cccccccc|c }
    	\hline
    	\multirow{2}{*}{Method} & \multicolumn{8}{c|}{Dataset} & \multirow{2}{*}{Average Ranking} \\
    	\cline{2-9}
    	\multirow{2}{*}{} & Covid & Optdigits & Statlog & Gesture & Spambase & Waveform & Galaxy & Beans & \\
    	\hline
    	MIFS \cite{battiti1994using} & 5.96 & 9.13 & 13.07 & 39.67 & 22.06 & 25.04 & 12.85 & \textbf{7.07} & 5.5 \\
    	FOU \cite{brown2009new} & 3.59 & 14.04 & 12.32 & 39.34 & 19.17 & 20.34 & 12.86 & 7.69 & 5.0 \\
    	MIM \cite{lewis1992feature} & 4.31 & 12.49 & 13.58 & 33.70 & 9.93 & 20.84 & 0.74 & 9.10 & 5.3 \\
    	MRMR \cite{peng2005feature} & 4.36 & 9.13 & 12.91 & 34.42 & 9.87 & 16.66 & 12.87 & 8.26 & 4.1 \\
    	JMI \cite{yang1999data} & 4.32 & 11.01 & 13.36 & 39.03 & 9.98 & \textbf{16.18} & 0.72 & 8.96 & 4.5 \\
    	CMIM \cite{fleuret2004fast} & 4.4 & \textbf{8.43} & 13.77 & 40.30 & \textbf{9.80} & 17.3 & 5.98 & 7.59 & 4.6 \\
    	\hline
    	RMI \cite{yu2019multivariate} & \textbf{3.86} & 9.50 & \textbf{13.35} & 20.28 & 9.82 & 19.98 & \textbf{0.51} & 7.16 & \textbf{2.8} \\
    	ARMI (Ours) & 4.36 & 8.91 & 13.36 & \textbf{20.06} & 10.71 & 20.98 & \textbf{0.51} & 7.32 & 3.8 \\
    	\hline
    \end{tabular}
\end{table*}

\section{Conclusion}
In this paper, we develop computationally efficient approximations for matrix-based R\'enyi's entropy. From the viewpoint of randomized linear algebra, we design random Gaussian approximation for integer-order matrix-based R\'enyi's entropy and use Taylor and Chebyshev series to approximate fractional matrix-based R\'enyi's entropy. Moreover, we exploit the structure of kernel matrices to design a memory-saving block low-rank approximation for calculating the matrix-based R\'enyi's entropy. Statistical guarantees are established for the proposed approximation algorithms. We also demonstrate their practical usage in real-world applications including feature selection and (multi-view) information bottleneck. In the future, we will try to acquire tighter theoretical upper bounds, and establish theoretical lower bounds for matrix-based R\'enyi's entropy. We will also continue to explore more practical applications of our fast matrix-based R\'enyi's entropy in both machine learning and neuroscience.

\section*{Acknowledgments}
This work has been supported by the Key Research and Development Program of China under Grant 2021ZD0110700; The National Natural Science Foundation of China under Grant 62106191; The Key Research and Development Program of Shaanxi Province (2021GXLH-Z-095); The Innovative Research Group of the National Natural Science Foundation of China (61721002); The consulting research project of the Chinese Academy of Engineering (The Online and Offline Mixed Educational Service System for The Belt and Road Training in MOOC China); Project of China Knowledge Centre for Engineering Science and Technology; The innovation team from the Ministry of Education (IRT\_17R86).

\appendix

\section*{Supplementary Experiment Results}
Additionally, we report the results for integer and fractional $\alpha$-order R\'enyi's entropy estimation with 3D plots in Fig. \ref{IntExp3D} and \ref{NonIntExp3D} for a more intuitive comparison.

\begin{figure}[t]
    \centering
    \includegraphics[scale=0.6]{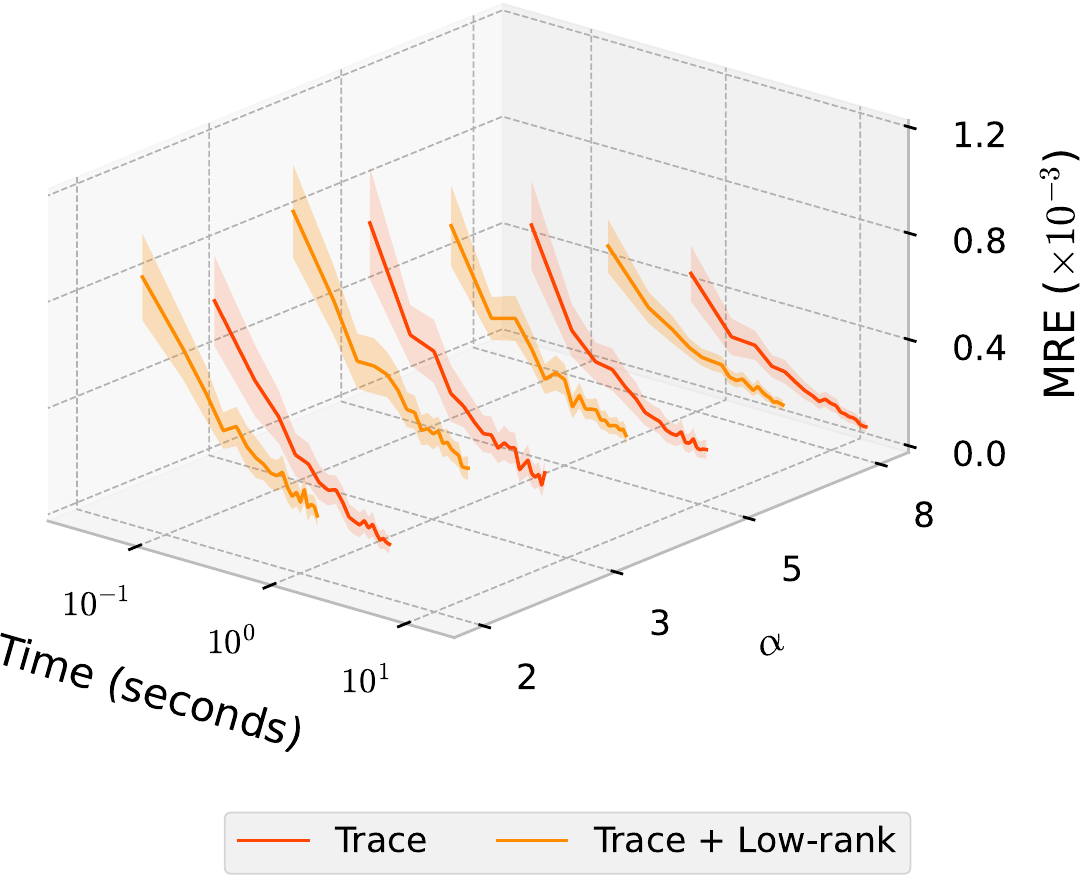}
    \caption{Time vs. MRE curves for integer $\alpha$-order R\'enyi's entropy estimation.}
    \label{IntExp3D}
\end{figure}

\begin{figure*}[t]
    \centering
    \includegraphics[scale=0.6]{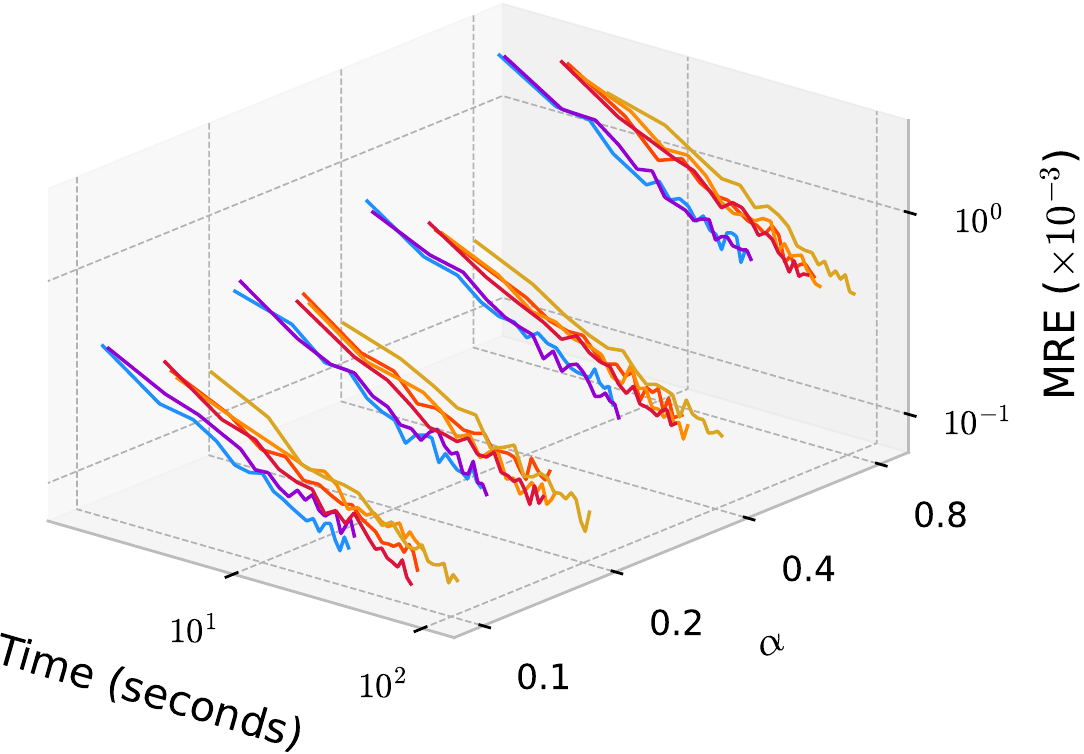}
    \includegraphics[scale=0.6]{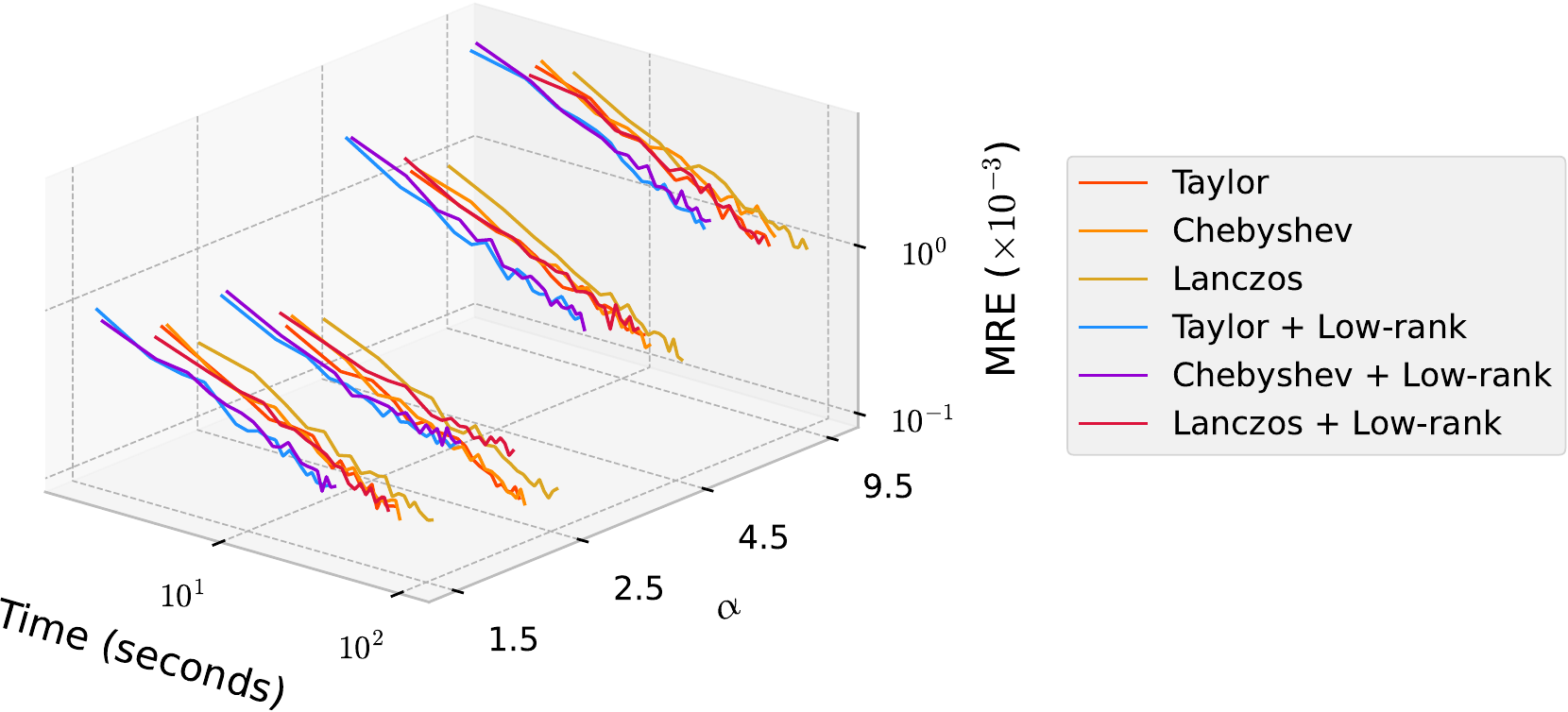}
    \caption{Time vs. MRE curves for fractional $\alpha$-order R\'enyi's entropy estimation.}
    \label{NonIntExp3D}
\end{figure*}

\subsection{Low-rank Approximation}
We then evaluate the performance of the proposed block low-rank approximation method compared with other low-rank approximation algorithms. We use the strategy above to generate the kernel matrix $G$ of size $1,000 \times 1,000$ with varying $\sigma$ in Gaussian kernel. For a fair comparison, we fix $c = 10$ and let $k$ varies from $10$ to $100$, and set the rank for other low-rank approximation methods as $q = c \cdot k$. We compare with the widely-adopted Nystr\"om method \cite{williams2000using} (Nystr\"om) and the best rank-$q$ approximation of $G$ obtained by SVD (SVD). The mean relative error (MRE) achieved for $\alpha = 1.5$ is shown in Figure \ref{LowRankExp}. As can be seen, our block low-rank approach achieves significantly higher accuracy than the Nystr\"om method, and is competitive with the best rank-$q$ approximation. This verifies the advantage of using our block low-rank method to approximate matrix-based R\'enyi's entropy.

\begin{figure}[t]
    \centering
    \includegraphics[width=0.48\textwidth]{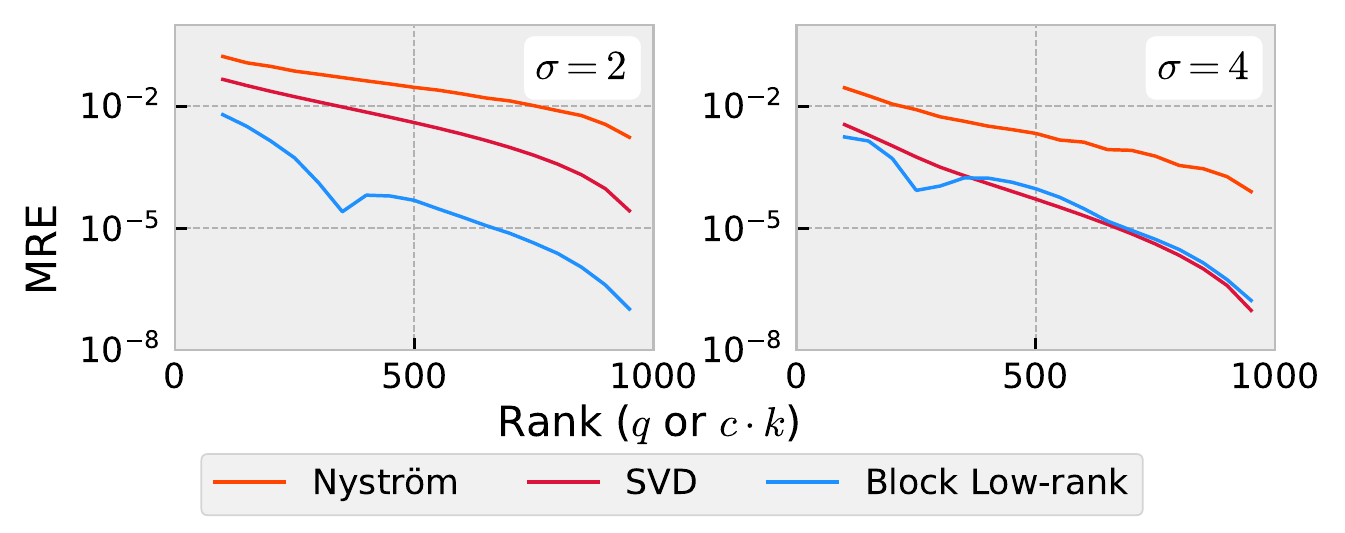}
    \caption{Rank vs. MRE curves for different low-rank approximation methods.}
    \label{LowRankExp}
\end{figure}

\subsection{Rank Deficient Approximation}
We further test the performance of Chebyshev and Lanczos approaches for fractional $\alpha$-order entropy estimation with $\lambda_{min} = 0$ and $\alpha = 0.5$, for which there exists a lower bound for $\epsilon$. To generate rank-deficient kernel matrices, we use the polynomial kernel $\varphi(\x_i, \x_j) = (\x_i^\top \x_j + 1)^2$ and set the dimension of data points $d = 138$, resulting in nearly $3\%$ of the eigenvalues being zero. We keep the other hyper-parameters the same as our previous fractional $\alpha$-order experiment, and the final Time vs. MRE curves are shown in Fig. \ref{DeficientExp}. As we can see, our approximation algorithms can still achieve satisfactory level of approximation accuracy even under this extreme circumstance.

\begin{figure}[t]
    \centering
    \includegraphics[scale=0.6]{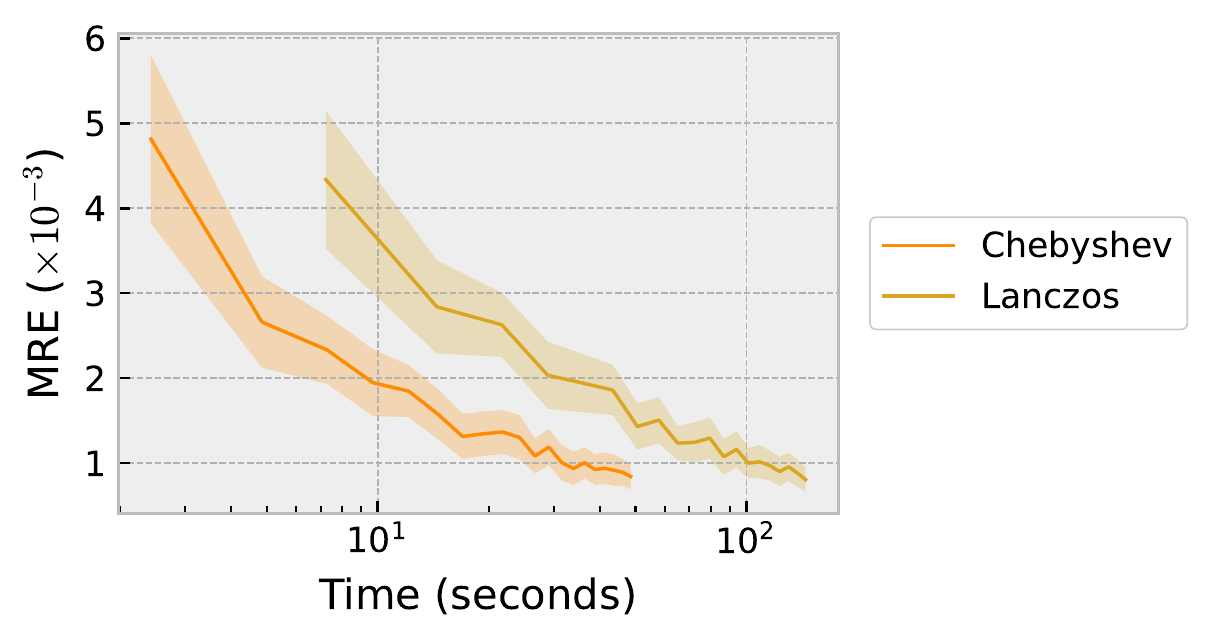}
    \caption{Time vs. MRE curves for rank deficient R\'enyi's entropy estimation.}
    \label{DeficientExp}
\end{figure}

\section*{Proof of Theorem \ref{th:ChebyBound} and \ref{th:LanczosBound}}
\subsection{Proof of Theorem \ref{th:ChebyBound}}

\begin{proof}
    Chebyshev series of the first kind is defined as $T_n(x) = \cos(n\arccos x)$ for $n \in \mathbb{N}$ on interval $[-1,1]$. Since the eigenvalues of $G$ all lie in $[0,\lambda_{max})$, the affine transformation
    \begin{align*}
    	F(x)=\frac{\lambda_{max}}{2}(x+1), \qquad \hat{T}_n(x)=T_n \circ F.
    \end{align*}
    allows us to define Chebyshev series on interval $[0,\lambda_{max}]$. The coefficients of Chebyshev series $\hat{T}_n$ can be calculated by:
    \begin{align*}
    	c_n &= \frac{2}{\pi} \int_0^\pi F^\alpha(cos\theta) \cos(n\theta) \dif\theta \\
    	&= \frac{2}{\pi} \int_0^\pi \prn*{\frac{\lambda_{max}}{2}(\cos\theta+1)}^\alpha \cos(n\theta) \dif\theta \\
    	&= \frac{2\lambda_{max}^\alpha \Gamma(\alpha+\frac{1}{2}) (\alpha)_n}{\sqrt{\pi} \Gamma(\alpha+1) (\alpha+n)_n},
    \end{align*}
    where $(\alpha)_n$ is the falling factorial: $(\alpha)_n = \alpha \cdot(\alpha-1)\cdot ... \cdot(\alpha-n+1)$. Let $f_m(x) = c_0/2 + \sum_{i=1}^m c_i \hat{T}_i(x)$ denote the Chebyshev series with degree $m$, then for each eigenvalue $\lambda$ of $G$:
    \begin{align}
    	&\quad\abs{\lambda^\alpha - f_m(\lambda)} = \abs*{\sum_{i=m+1}^\infty c_i \tilde{T}_i(\lambda)} \nonumber\\
    	&\le \sum_{i=m+1}^\infty \abs{c_i} = \sum_{i=m+1}^\infty \abs*{\frac{2\lambda_{max}^\alpha \Gamma(\alpha + \frac{1}{2}) (\alpha)_i}{\sqrt{\pi} \Gamma(\alpha + 1) (\alpha+i)_i}} \label{ChebyProp1}\\
    	&= \frac{2\lambda_{max}^\alpha}{\sqrt{\pi}}\sum_{i=m+1}^\infty \abs*{\frac{\Gamma(\alpha + \frac{1}{2}) \Gamma(\alpha + 1)}{\Gamma(\alpha + i + 1) \Gamma(\alpha - i + 1)}} \nonumber\\
    	&\le \frac{2\lambda_{max}^\alpha}{\sqrt{\pi}}\sum_{i=m+1}^\infty \abs*{\frac{\Gamma(\alpha + \frac{1}{2}) \Gamma(\alpha + 1)}{\Gamma(i-\alpha)\Gamma(\alpha-i+1) (i-\alpha)^{2\alpha+1}}} \label{GammaIneq2}\\
    	&\le \frac{2\lambda_{max}^\alpha \Gamma(\alpha + \frac{1}{2}) \Gamma(\alpha + 1)}{\pi^{3/2}}\sum_{i=m+1}^\infty \abs*{\frac{1}{(i-\alpha)^{2\alpha+1}}} \label{GammaProp2}\\
    	&\le \frac{2\lambda_{max}^\alpha \Gamma(\alpha + \frac{1}{2}) \Gamma(\alpha + 1)}{\pi^{3/2}} \int_m^\infty \frac{1}{(x-\alpha)^{2\alpha+1}} \dif x \label{IntPoly}\\
    	&= \frac{2\lambda_{max}^\alpha \Gamma(\alpha + \frac{1}{2}) \Gamma(\alpha + 1)}{\pi^{3/2}} \frac{1}{2\alpha(m-\alpha)^{2\alpha}} \nonumber\\
    	&= \frac{\lambda_{max}^\alpha \Gamma(\alpha + \frac{1}{2}) \Gamma(\alpha)}{\pi^{3/2} (m-\alpha)^{2\alpha}}. \nonumber
    \end{align}
    (\ref{ChebyProp1}) follows by noticing that $\hat{T}_n(x) \in [-1,1]$ for any $x \in [0,\lambda_{max}]$. (\ref{GammaIneq2}) follows by applying Lemma \ref{GammaIneq} on $R(i-\alpha,2\alpha+1)$ similar to (\ref{GammaIneq1}). (\ref{GammaProp2}) follows by Euler's reflection formula similar to (\ref{GammaProp1}). (\ref{IntPoly}) follows by noticing that $n^{-k} \le \int_{n-1}^n x^{-k} \dif x$ for $n > 1$ and $k > 1$. If $G$ has full rank, i.e, $\lambda_{min} > 0$, then by choosing $m$ as:
    \begin{equation*}
    	m \ge \alpha + \sqrt{\kappa} \sqrt[2\alpha]{\frac{\Gamma(\alpha + \frac{1}{2}) \Gamma(\alpha)}{\epsilon\pi^{3/2}}},
    \end{equation*}
    we have $\abs{\lambda^\alpha - f_m(\lambda)} \le \epsilon \lambda_{min}^\alpha \le \epsilon \lambda^\alpha$, which yields the same upper bound as (\ref{PolyBound}).
    
    Similarly, if $G$ is rank deficient, i.e, $\lambda_{min} = 0$, by choosing $m$ as:
    \begin{equation*}
    	m \ge \alpha + \sqrt[2\alpha]{\frac{n\Gamma(\alpha + \frac{1}{2}) \Gamma(\alpha)}{\epsilon\pi^{3/2}}},
    \end{equation*}
    we have $\abs{\lambda^\alpha - f_m(\lambda)} \le \frac{\epsilon}{n} \lambda_{max}^\alpha$, which yields the same error upper bound as above.
\end{proof}

\subsection{Proof of Theorem \ref{th:LanczosBound}}

\begin{proof}
    We apply the result of \cite{ubaru2017fast} to establish the error bound of Lanczos iteration, as shown in the following lemma:
    \begin{lemma} \cite{ubaru2017fast}
        Consider a symmetric positive definite matrix $A \in \mathbb{R}^{n \times n}$ with eigenvalues in $[\lambda_{min}, \lambda_{max}]$ and condition number $\kappa = \lambda_{max}/\lambda_{min}$, and let $f$ be a function that is analytic and either positive or negative inside its interval, and whose absolute maximum and minimum values in the interval are $M_\rho$ and $m_\rho$ respectively. Let $\epsilon$, $\eta$ be constants in the interval $(0, 1)$. Then for stochastic Lanczos quadrature parameters satisfying:
        \begin{enumerate}
            \item $s \ge \frac{24}{\epsilon^2} \ln(2/\eta)$ number of random Rademacher vectors, and
            \item $m \ge \frac{\sqrt{\kappa}}{4} \log(K/\epsilon)$ number of Lanczos steps,
        \end{enumerate}
        where $K = \frac{(\lambda_{max}-\lambda_{min}) (\sqrt{\kappa}-1)^2 M_\rho}{\sqrt{\kappa} m_\rho}$, the output $\Gamma$ is such that:
        \begin{equation*}
            \mathbb{P}\brk*{|\tr(f(A)) - \Gamma| \le \epsilon |\tr(f(A))|} \ge 1-\eta.
        \end{equation*}
    \end{lemma}
    
    For our problem, $\lambda_{min} \ge 0$ and $\lambda_{max} \le 1$, thus $\lambda_{max} - \lambda_{min} \le 1$. The function $f$ to be approximated is the $\alpha$-power function $f(A) = A^\alpha$, therefore $M_\rho/m_\rho = \lambda_{max}^\alpha / \lambda_{min}^\alpha = \kappa^\alpha$. From the analysis above, we have the final upper bound for the number of Lanczos steps: $m = \ceil*{\frac{\sqrt{\kappa}}{4} \log(\kappa^{\alpha+\frac{1}{2}}/\epsilon)}$.
\end{proof}

\bibliographystyle{IEEEtran}
\bibliography{Entropy_TSP}

\begin{thebibliography}{10}
\providecommand{\url}[1]{#1}
\csname url@samestyle\endcsname
\providecommand{\newblock}{\relax}
\providecommand{\bibinfo}[2]{#2}
\providecommand{\BIBentrySTDinterwordspacing}{\spaceskip=0pt\relax}
\providecommand{\BIBentryALTinterwordstretchfactor}{4}
\providecommand{\BIBentryALTinterwordspacing}{\spaceskip=\fontdimen2\font plus
\BIBentryALTinterwordstretchfactor\fontdimen3\font minus
  \fontdimen4\font\relax}
\providecommand{\BIBforeignlanguage}[2]{{%
\expandafter\ifx\csname l@#1\endcsname\relax
\typeout{** WARNING: IEEEtran.bst: No hyphenation pattern has been}%
\typeout{** loaded for the language `#1'. Using the pattern for}%
\typeout{** the default language instead.}%
\else
\language=\csname l@#1\endcsname
\fi
#2}}
\providecommand{\BIBdecl}{\relax}
\BIBdecl

\bibitem{tishby1999information}
N.~Tishby, F.~C. Pereira, and W.~Bialek, ``The information bottleneck method,''
  in \emph{Proc. 37th Annual Allerton Conference on Communications, Control and
  Computing, 1999}, 1999, pp. 368--377.

\bibitem{jaynes1957information}
E.~T. Jaynes, ``Information theory and statistical mechanics,'' \emph{Physical
  review}, vol. 106, no.~4, p. 620, 1957.

\bibitem{jaynes1957information2}
------, ``Information theory and statistical mechanics. ii,'' \emph{Physical
  review}, vol. 108, no.~2, p. 171, 1957.

\bibitem{principe2010information}
J.~C. Principe, \emph{Information theoretic learning: Renyi's entropy and
  kernel perspectives}.\hskip 1em plus 0.5em minus 0.4em\relax Springer Science
  \& Business Media, 2010.

\bibitem{victor2006approaches}
J.~D. Victor, ``Approaches to information-theoretic analysis of neural
  activity,'' \emph{Biological theory}, vol.~1, no.~3, pp. 302--316, 2006.

\bibitem{timme2018tutorial}
N.~M. Timme and C.~Lapish, ``A tutorial for information theory in
  neuroscience,'' \emph{eneuro}, vol.~5, no.~3, 2018.

\bibitem{battiti1994using}
{R. Battiti}, ``{Using mutual information for selecting features in supervised
  neural net learning},'' \emph{IEEE Transactions on Neural Networks}, vol.
  5(4), no.~4, pp. 537--550, 1994.

\bibitem{williams2010nonnegative}
P.~L. Williams and R.~D. Beer, ``Nonnegative decomposition of multivariate
  information,'' \emph{arXiv preprint arXiv:1004.2515}, 2010.

\bibitem{wibral2017partial}
M.~Wibral, V.~Priesemann, J.~W. Kay, J.~T. Lizier, and W.~A. Phillips,
  ``Partial information decomposition as a unified approach to the
  specification of neural goal functions,'' \emph{Brain and cognition}, vol.
  112, pp. 25--38, 2017.

\bibitem{alemi2016deep}
A.~A. Alemi, I.~Fischer, J.~V. Dillon, and K.~Murphy, ``Deep variational
  information bottleneck,'' in \emph{International Conference on Learning
  Representations}, 2017.

\bibitem{RDevonHjelm2018LearningDR}
R.~D. Hjelm, A.~Fedorov, S.~Lavoie-Marchildon, K.~Grewal, P.~Bachman,
  A.~Trischler, and Y.~Bengio, ``Learning deep representations by mutual
  information estimation and maximization,'' \emph{international conference on
  learning representations}, 2018.

\bibitem{xu2017information}
A.~Xu and M.~Raginsky, ``Information-theoretic analysis of generalization
  capability of learning algorithms,'' in \emph{Advances in Neural Information
  Processing Systems}, vol.~30, 2017.

\bibitem{fan2006statistical}
J.~Fan and R.~Li, ``Statistical challenges with high dimensionality,'' in
  \emph{Proceedings of the international Congress of Mathematicians}, 2006.

\bibitem{giraldo2014measures}
L.~G.~S. Giraldo, M.~Rao, and J.~C. Principe, ``Measures of entropy from data
  using infinitely divisible kernels,'' \emph{IEEE Transactions on Information
  Theory}, vol.~61, no.~1, pp. 535--548, 2014.

\bibitem{giraldo2014information}
L.~G.~S. Giraldo and J.~C. Principe, ``Information theoretic learning with
  infinitely divisible kernels,'' in \emph{Internatonal Conference on Learning
  Representation}, 2014.

\bibitem{yu2019multivariate}
S.~Yu, L.~G.~S. Giraldo, R.~Jenssen, and J.~C. Principe, ``Multivariate
  extension of matrix-based r\'enyi alpha-order entropy functional,''
  \emph{IEEE Transactions on Pattern Analysis and Machine Intelligence},
  vol.~42, no.~11, pp. 2960--2966, 2019.

\bibitem{shannon1948mathematical}
C.~E. Shannon, ``A mathematical theory of communication,'' \emph{The Bell
  system technical journal}, vol.~27, no.~3, pp. 379--423, 1948.

\bibitem{renyi1961measures}
A.~R{\'e}nyi, ``On measures of entropy and information,'' in \emph{Proceedings
  of the Fourth Berkeley Symposium on Mathematical Statistics and Probability,
  Volume 1: Contributions to the Theory of Statistics}, vol.~4.\hskip 1em plus
  0.5em minus 0.4em\relax University of California Press, 1961, pp. 547--562.

\bibitem{brockmeier2017quantifying}
A.~J. Brockmeier, T.~Mu, S.~Ananiadou, and J.~Y. Goulermas, ``Quantifying the
  informativeness of similarity measurements,'' \emph{Journal of Machine
  Learning Research}, vol.~18, pp. 1--61, 2017.

\bibitem{alvarez2017kernel}
\BIBentryALTinterwordspacing
A.~M. Alvarez-Meza, J.~A. Lee, M.~Verleysen, and G.~Castellanos-Dominguez,
  ``{Kernel-based dimensionality reduction using Renyi's $\alpha$-entropy
  measures of similarity},'' \emph{Neurocomputing}, vol. 222, no. April 2016,
  pp. 36--46, 2017. [Online]. Available:
  \url{http://dx.doi.org/10.1016/j.neucom.2016.10.004}
\BIBentrySTDinterwordspacing

\bibitem{sarvani2022hrel}
C.~Sarvani, M.~Ghorai, S.~R. Dubey, and S.~S. Basha, ``Hrel: Filter pruning
  based on high relevance between activation maps and class labels,''
  \emph{Neural Networks}, vol. 147, pp. 186--197, 2022.

\bibitem{zheng2022brainib}
K.~Zheng, S.~Yu, B.~Li, R.~Jenssen, and B.~Chen, ``Brainib: Interpretable brain
  network-based psychiatric diagnosis with graph information bottleneck,''
  \emph{arXiv preprint arXiv:2205.03612}, 2022.

\bibitem{de2019data}
I.~De~La Pava~Panche, A.~M. Alvarez-Meza, and A.~Orozco-Gutierrez, ``A
  data-driven measure of effective connectivity based on renyi's
  $\alpha$-entropy,'' \emph{Frontiers in neuroscience}, vol.~13, p. 1277, 2019.

\bibitem{elad2010sparse}
M.~Elad, \emph{Sparse and redundant representations: from theory to
  applications in signal and image processing}.\hskip 1em plus 0.5em minus
  0.4em\relax Springer, 2010.

\bibitem{li2014large}
M.~Li, W.~Bi, J.~T. Kwok, and B.-L. Lu, ``Large-scale nystr{\"o}m kernel matrix
  approximation using randomized svd,'' \emph{IEEE Transactions on Neural
  Networks and Learning Systems}, vol.~26, no.~1, pp. 152--164, 2014.

\bibitem{mahoney2009cur}
M.~W. Mahoney and P.~Drineas, ``Cur matrix decompositions for improved data
  analysis,'' \emph{Proceedings of the National Academy of Sciences}, vol. 106,
  no.~3, pp. 697--702, 2009.

\bibitem{mahoney2011randomized}
M.~W. Mahoney, ``Randomized algorithms for matrices and data,''
  \emph{Foundations and Trends in Machine Learning}, vol.~3, no.~2, 2011.

\bibitem{watkins2008qr}
D.~S. Watkins, ``The qr algorithm revisited,'' \emph{SIAM review}, vol.~50,
  no.~1, pp. 133--145, 2008.

\bibitem{watkins2004fundamentals}
------, \emph{Fundamentals of matrix computations}.\hskip 1em plus 0.5em minus
  0.4em\relax John Wiley \& Sons, 2004, vol.~64.

\bibitem{avron2011randomized}
H.~Avron and S.~Toledo, ``Randomized algorithms for estimating the trace of an
  implicit symmetric positive semi-definite matrix,'' \emph{Journal of the ACM
  (JACM)}, vol.~58, no.~2, pp. 1--34, 2011.

\bibitem{klir2013uncertainty}
G.~J. Klir and M.~J. Wierman, \emph{Uncertainty-based information: elements of
  generalized information theory}.\hskip 1em plus 0.5em minus 0.4em\relax
  Physica, 2013, vol.~15.

\bibitem{bhatia2006infinitely}
R.~Bhatia, ``Infinitely divisible matrices,'' \emph{The American Mathematical
  Monthly}, vol. 113, no.~3, pp. 221--235, 2006.

\bibitem{roosta2015improved}
\BIBentryALTinterwordspacing
F.~Roosta-Khorasani and U.~Ascher, ``{Improved Bounds on Sample Size for
  Implicit Matrix Trace Estimators},'' \emph{Foundations of Computational
  Mathematics}, vol.~15, no.~5, pp. 1187--1212, 2015. [Online]. Available:
  \url{http://dx.doi.org/10.1007/s10208-014-9220-1}
\BIBentrySTDinterwordspacing

\bibitem{boutsidis2017randomized}
C.~Boutsidis, P.~Drineas, P.~Kambadur, E.-M. Kontopoulou, and A.~Zouzias, ``A
  randomized algorithm for approximating the log determinant of a symmetric
  positive definite matrix,'' \emph{Linear Algebra and its Applications}, vol.
  533, pp. 95--117, 2017.

\bibitem{das2019}
S.~Das, ``{Inequalities for q-gamma function ratios},'' \emph{Analysis and
  Mathematical Physics}, vol.~9, no.~1, pp. 313--321, 2019.

\bibitem{clenshaw1955note}
C.~W. Clenshaw, ``A note on the summation of chebyshev series,''
  \emph{Mathematics of Computation}, vol.~9, no.~51, pp. 118--120, 1955.

\bibitem{ubaru2017fast}
S.~Ubaru, J.~Chen, and Y.~Saad, ``Fast estimation of tr(f(a)) via stochastic
  lanczos quadrature,'' \emph{SIAM Journal on Matrix Analysis and
  Applications}, vol.~38, no.~4, pp. 1075--1099, 2017.

\bibitem{li2005}
W.~Li and W.~Sun, ``{The perturbation bounds for eigenvalues of normal
  matrices},'' \emph{Numerical Linear Algebra with Applications}, vol.~12, no.
  2-3, pp. 89--94, 2005.

\bibitem{williams2000using}
C.~Williams and M.~Seeger, ``Using the nystr{\"o}m method to speed up kernel
  machines,'' \emph{Advances in neural information processing systems},
  vol.~13, 2000.

\bibitem{sisi2017memory}
S.~Si, C.-J. Hsieh, and I.~S. Dhillon, ``Memory efficient kernel
  approximation,'' \emph{Journal of Machine Learning Research}, vol.~18,
  no.~20, pp. 1--32, 2017.

\bibitem{halko2011}
N.~Halko, P.~G. Martinsson, and J.~A. Tropp, ``{Finding structure with
  randomness: Probabilistic algorithms for constructing approximate matrix
  decompositions},'' \emph{SIAM Review}, vol.~53, no.~2, pp. 217--288, 2011.

\bibitem{opencv_library}
G.~Bradski, ``{The OpenCV Library},'' \emph{Dr. Dobb's Journal of Software
  Tools}, 2000.

\bibitem{eigenweb}
G.~Guennebaud and B.~Jacob, ``Eigen v3,'' http://eigen.tuxfamily.org, 2010.

\bibitem{yeung1991new}
R.~W. Yeung, ``A new outlook on shannon's information measures,'' \emph{IEEE
  transactions on information theory}, vol.~37, no.~3, pp. 466--474, 1991.

\bibitem{yu2021deep}
X.~Yu, S.~Yu, and J.~C. Principe, ``{Deep Deterministic Information Bottleneck
  with Matrix-Based Entropy Functional},'' \emph{IEEE International Conference
  on Acoustics, Speech and Signal Processing (ICCASP)}, pp. 3160--3164, 2021.

\bibitem{ardizzone2020training}
L.~Ardizzone, R.~Mackowiak, C.~Rother, and U.~K{\"{o}}the, ``{Training
  normalizing flows with the information bottleneck for competitive generative
  classification},'' \emph{Advances in Neural Information Processing Systems},
  vol. 2020-December, no.~1, pp. 1--13, 2020.

\bibitem{wu2020graph}
T.~Wu, H.~Ren, P.~Li, and J.~Leskovec, ``{Graph information bottleneck},''
  \emph{Advances in Neural Information Processing Systems}, vol. 2020-December,
  no. NeurIPS, 2020.

\bibitem{kolchinsky2019nonlinear}
A.~Kolchinsky, B.~D. Tracey, and D.~H. Wolpert, ``Nonlinear information
  bottleneck,'' \emph{Entropy}, vol.~21, no.~12, p. 1181, 2019.

\bibitem{belghazi2018mutual}
M.~I. Belghazi, A.~Baratin, S.~Rajeshwar, S.~Ozair, Y.~Bengio, A.~Courville,
  and D.~Hjelm, ``Mutual information neural estimation,'' in
  \emph{International Conference on Machine Learning}.\hskip 1em plus 0.5em
  minus 0.4em\relax PMLR, 2018, pp. 531--540.

\bibitem{elad2019direct}
A.~Elad, D.~Haviv, Y.~Blau, and T.~Michaeli, ``Direct validation of the
  information bottleneck principle for deep nets,'' in \emph{Proceedings of the
  IEEE/CVF International Conference on Computer Vision Workshops}, 2019, pp.
  0--0.

\bibitem{simonyan2015very}
K.~Simonyan and A.~Zisserman, ``{Very deep convolutional networks for
  large-scale image recognition},'' \emph{3rd International Conference on
  Learning Representations, ICLR 2015 - Conference Track Proceedings}, pp.
  1--14, 2015.

\bibitem{wang2019deep}
Q.~Wang, C.~Boudreau, Q.~Luo, P.-N. Tan, and J.~Zhou, ``Deep multi-view
  information bottleneck,'' in \emph{Proceedings of the 2019 SIAM International
  Conference on Data Mining}.\hskip 1em plus 0.5em minus 0.4em\relax SIAM,
  2019, pp. 37--45.

\bibitem{aguerri2019distributed}
I.~E. Aguerri and A.~Zaidi, ``Distributed variational representation
  learning,'' \emph{IEEE transactions on pattern analysis and machine
  intelligence}, vol.~43, no.~1, pp. 120--138, 2019.

\bibitem{zhang2022multi}
Q.~Zhang, S.~Yu, J.~Xin, and B.~Chen, ``Multi-view information bottleneck
  without variational approximation,'' in \emph{ICASSP 2022-2022 IEEE
  International Conference on Acoustics, Speech and Signal Processing
  (ICASSP)}.\hskip 1em plus 0.5em minus 0.4em\relax IEEE, 2022, pp. 4318--4322.

\bibitem{samuel2018dont}
\BIBentryALTinterwordspacing
S.~L. Smith, P.-J. Kindermans, and Q.~V. Le, ``Don't decay the learning rate,
  increase the batch size,'' in \emph{International Conference on Learning
  Representations}, 2018. [Online]. Available:
  \url{https://openreview.net/forum?id=B1Yy1BxCZ}
\BIBentrySTDinterwordspacing

\bibitem{uci}
\BIBentryALTinterwordspacing
D.~Dua and C.~Graff, ``{UCI} machine learning repository,'' 2017. [Online].
  Available: \url{http://archive.ics.uci.edu/ml}
\BIBentrySTDinterwordspacing

\bibitem{koklu2020multiclass}
\BIBentryALTinterwordspacing
M.~Koklu and I.~A. Ozkan, ``{Multiclass classification of dry beans using
  computer vision and machine learning techniques},'' \emph{Computers and
  Electronics in Agriculture}, vol. 174, no. May, p. 105507, 2020. [Online].
  Available: \url{https://doi.org/10.1016/j.compag.2020.105507}
\BIBentrySTDinterwordspacing

\bibitem{vinh2016can}
N.~X. Vinh, S.~Zhou, J.~Chan, and J.~Bailey, ``{Can high-order dependencies
  improve mutual information based feature selection?}'' \emph{Pattern
  Recognition}, vol.~53, pp. 46--58, 2016.

\bibitem{brown2009new}
G.~Brown, ``{A new perspective for information theoretic feature selection},''
  \emph{Journal of Machine Learning Research}, vol.~5, no.~1, pp. 49--56, 2009.

\bibitem{lewis1992feature}
D.~D. Lewis, ``{Feature selection and feature extraction for text
  categorization},'' \emph{Proceedings of the workshop on Speech and Natural
  Language}, p. 212, 1992.

\bibitem{peng2005feature}
H.~Peng, F.~Long, and C.~Ding, ``Feature selection based on mutual information
  criteria of max-dependency, max-relevance, and min-redundancy,'' \emph{IEEE
  Transactions on pattern analysis and machine intelligence}, vol.~27, no.~8,
  pp. 1226--1238, 2005.

\bibitem{yang1999data}
H.~Yang, ``{Data visualization and feature selection: New algorithms for
  nongaussian data},'' \emph{Advances in Neural Information Processing
  Systems}, no.~Mi, pp. 687----693, 1999.

\bibitem{fleuret2004fast}
F.~Fleuret, ``Fast binary feature selection with conditional mutual
  information.'' \emph{Journal of Machine learning research}, vol.~5, no.~9,
  2004.

\bibitem{vinh2014reconsidering}
N.~X. Vinh, J.~Chan, and J.~Bailey, ``{Reconsidering mutual information based
  feature selection: A statistical significance view},'' \emph{Proceedings of
  the National Conference on Artificial Intelligence}, vol.~3, no.~1, pp.
  2092--2098, 2014.

\bibitem{chang2011libsvm}
C.-C. Chang and C.-J. Lin, ``Libsvm: A library for support vector machines,''
  \emph{ACM transactions on intelligent systems and technology (TIST)}, vol.~2,
  no.~3, pp. 1--27, 2011.

\end{thebibliography}

\begin{IEEEbiography}[{\includegraphics[width=1in,height=1.25in,clip,keepaspectratio]{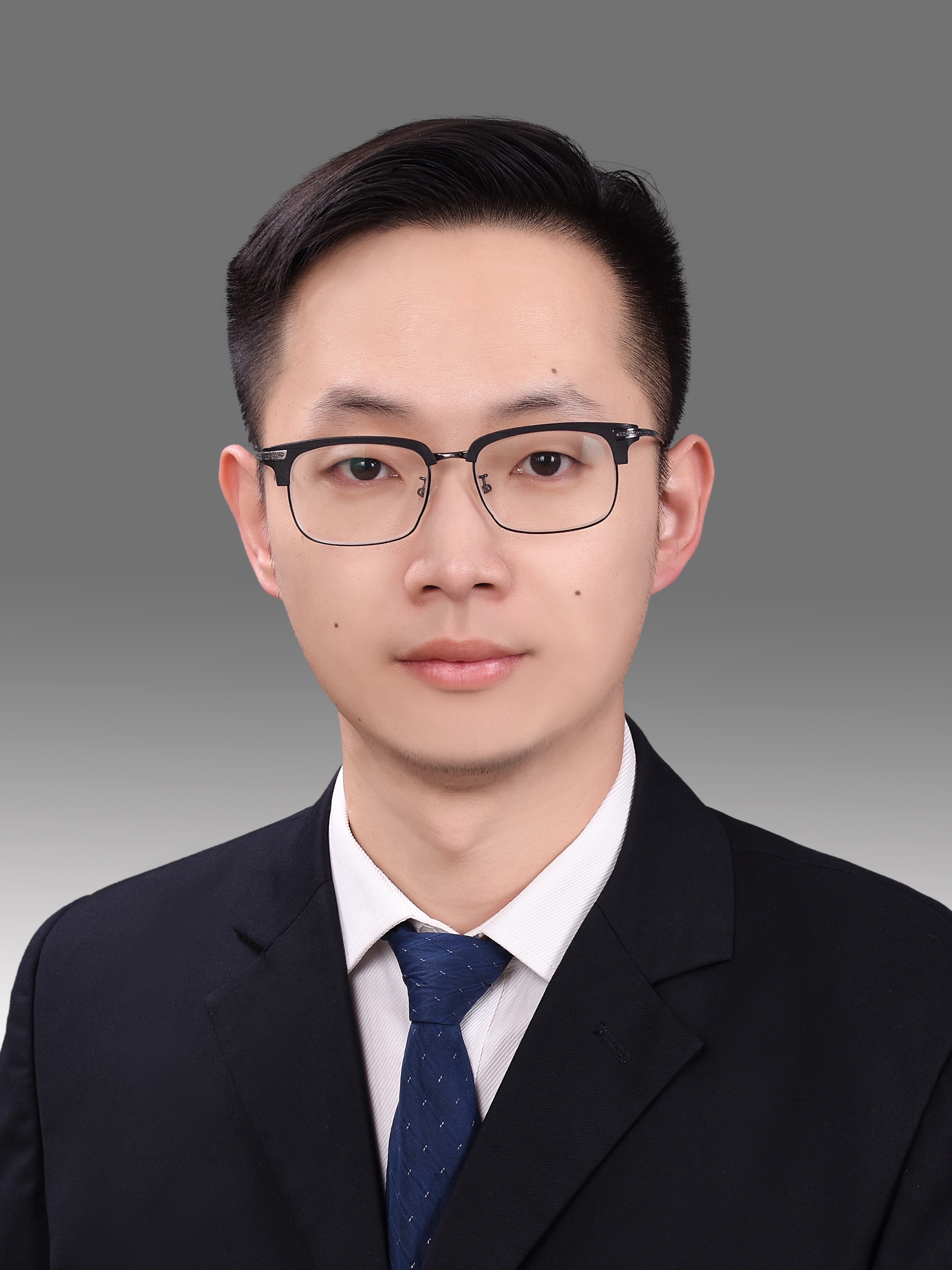}}]{Tieliang Gong}
	received the Ph.D. degree from Xi'an Jiaotong University,  Xi'an China, in 2018. From Sep 2018 to  Oct. 2020, he was a Post-Doctoral Researcher with the Department of Mathematics and Statistics, University of Ottawa, Ottawa, ON, Canada. He is currently an Assistant Professor with the School of Computer Science and Technology, Xi'an Jiaotong University, Xi'an. His current research interests include statistical learning theory, machine learning and information theory.
\end{IEEEbiography}
\begin{IEEEbiography}[{\includegraphics[width=1in,height=1.25in,clip,keepaspectratio]{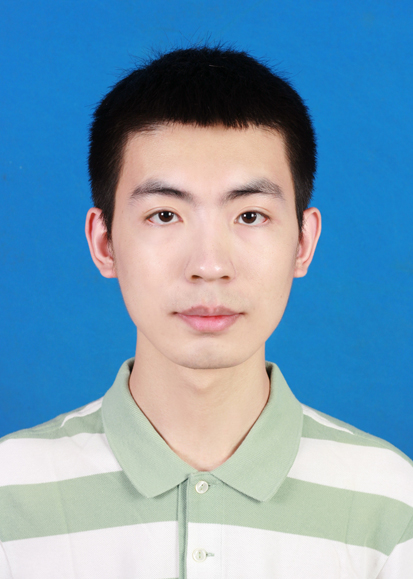}}]{Yuxin Dong}
	received the B.S. degree from Xi'an Jiaotong University, Xi'an, China in 2019. He is currently a Ph.D. student with the School of Computer Science and Technology, Xi'an Jiaotong University. His research interests include information theory, statistical learning theory and bioinformatics.
\end{IEEEbiography}
\begin{IEEEbiography}[{\includegraphics[width=1in,height=1.25in,clip,keepaspectratio]{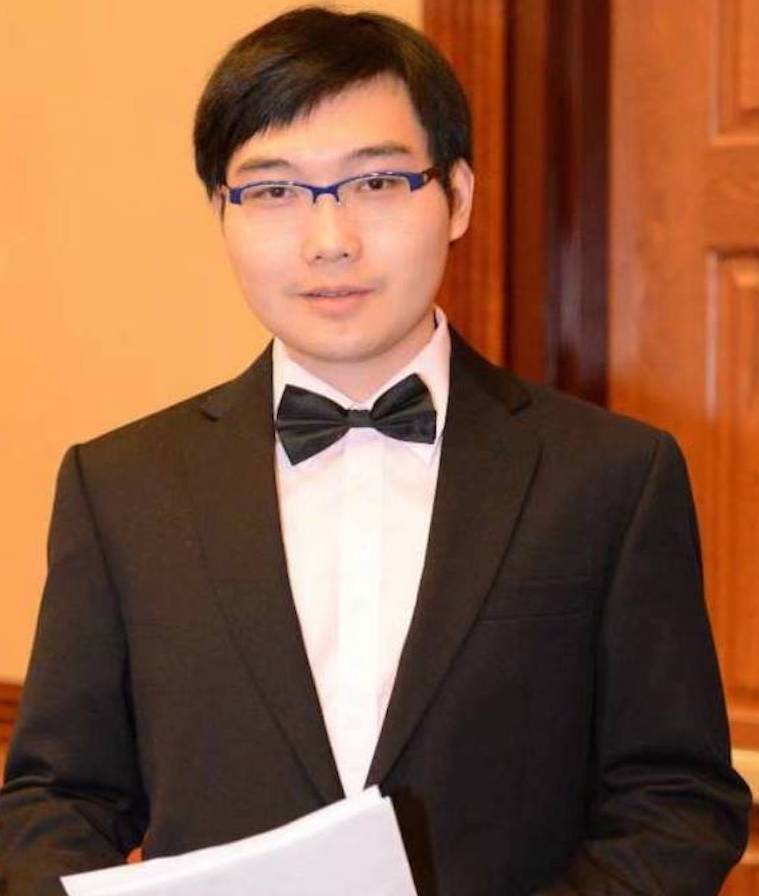}}]{Dr. Shujian Yu}
	will join the Department of Computer Science at the Vrije Universiteit Amsterdam in 2023 February as a tenure-track assistant professor. He is also affiliated with the Department of Physics and Technology at the UiT – The Arctic University of Norway. He was a machine learning research scientist at the NEC Labs Europe from 2019 to 2021. He received his Ph.D. degree in Electrical and Computer Engineering from the University of Florida in 2019, with a Ph.D. minor in Statistics. He received the 2020 International Neural Networks Society (INNS) Aharon Katzir Young Investigator Award for the contribution on the development of novel information theoretic measures for analysis and training of deep neural networks. He is also selected for the 2023 AAAI New Faculty Highlights.
\end{IEEEbiography}
\begin{IEEEbiography}[{\includegraphics[width=1in,height=1.25in,clip,keepaspectratio]{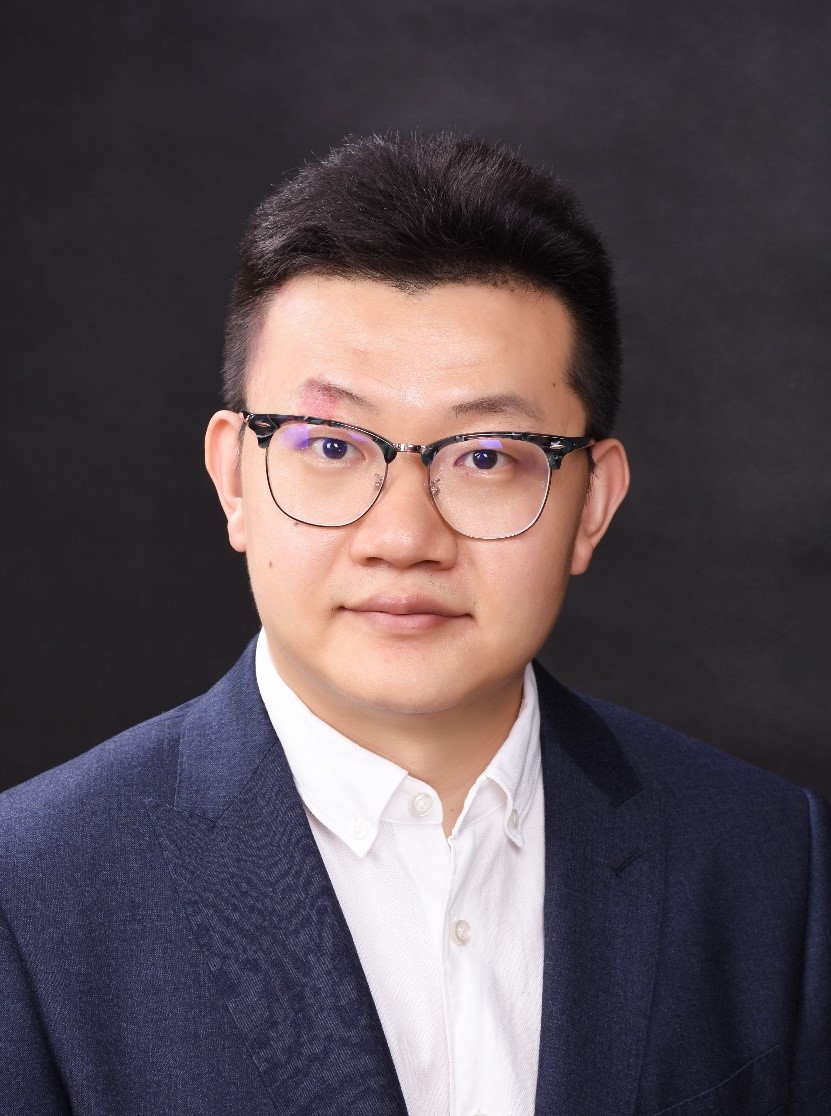}}]{Bo Dong}
	received his Ph.D. degree in computer science and technology from Xi’an Jiaotong University in 2014. He did postdoctoral research in the MOE Key Lab for Intelligent Networks and Network Security, Xi’an Jiaotong University from 2014 to 2017. He currently serves as the research director in the School of Continuing Education, Xi’an Jiaotong University. His research interests focus on data mining and intelligent e-Learning.
\end{IEEEbiography}

\end{document}